\documentclass[accepted]{uai2026} 
                        
\usepackage[american]{babel}

\usepackage[natbib, style=authoryear-comp, maxcitenames=2, uniquelist=false]{biblatex}
\let\cite\citep
\usepackage{amsfonts,amsmath,amssymb,amsthm}
\usepackage{algorithm}
\usepackage{algpseudocode}
\usepackage{booktabs}
\usepackage{graphicx} 
\usepackage{appendix}
\usepackage{xcolor}
\usepackage{bbm}
\usepackage{multicol}

\DeclareMathOperator*{\argmax}{arg\,max}
\DeclareMathOperator*{\argmin}{arg\,min}
\newcommand{\iou}{\textnormal{IoU}} 
\newcommand{\xt}{X^{\textnormal{test}}} 
\newcommand{\yt}{Y^{\textnormal{test}}}

\newcommand{\cx}{\mathcal{X}} 
\newcommand{\cq}{\mathcal{Q}} 
\newcommand{\cy}{\mathcal{Y}} 
\newcommand{\ca}{\mathcal{A}} 
\newcommand{\e}{\mathbb{E}}

\theoremstyle{plain}  
\newtheorem{theorem}{Theorem}[section]
\newtheorem{lemma}{Lemma}[section]

\newtheorem{proposition}{Proposition}[section]
\theoremstyle{remark}

\newtheorem{assumption}{Assumption}

\usepackage{chngcntr}
\usepackage{hyperref}
\hypersetup{colorlinks=true, urlcolor=blue, linkcolor=red}

\title{Conformal Prediction Sets for Instance Segmentation}

\author[1,2]{\href{mailto:kerrilu@mit.edu}{Kerri Lu}{}}
\author[4]{\href{mailto:dkluger@mit.edu}{Dan M. Kluger}{}}
\author[1,2]{\href{mailto:stephenbates@mit.edu}{Stephen Bates}{}}
\author[1,3,4]{\href{mailto:sherwang@mit.edu}{Sherrie Wang}{}}
\affil[1]{Laboratory for Information and Decision Systems\\MIT\\Cambridge, MA, USA}
\affil[2]{Department of Electrical Engineering and Computer Science\\MIT\\Cambridge, MA, USA}
\affil[3]{Department of Mechanical Engineering\\MIT\\Cambridge, MA, USA}
\affil[4]{Institute for Data, Systems, and Society\\MIT\\Cambridge, MA, USA}
  
\begin{document}
\maketitle

\begin{abstract} 
Current instance segmentation models achieve high performance on average predictions, but lack principled uncertainty quantification: their outputs are not calibrated, and there is no guarantee that a predicted mask is close to the ground truth.
To address this limitation, we introduce a conformal prediction algorithm to generate adaptive confidence sets for instance segmentation. 
Given an image and a pixel coordinate query, our algorithm generates a confidence set of instance predictions for that pixel, with a provable guarantee for the probability that at least one of the predictions has high Intersection-Over-Union (IoU) with the true object instance mask. We apply our algorithm to instance segmentation examples in agricultural field delineation, cell segmentation, and vehicle detection. Empirically, we find that our prediction sets vary in size based on query difficulty and attain the target coverage, outperforming baselines (naive best parameter and morphological dilation-based methods). We provide versions of the algorithm with asymptotic and finite sample guarantees. Our work is the first to capture structural uncertainty in instance segmentation by constructing confidence sets of diverse segmentation predictions.
\end{abstract}

\section{INTRODUCTION}

\begin{figure*}
  \centering
  \includegraphics[width=0.7\textwidth]{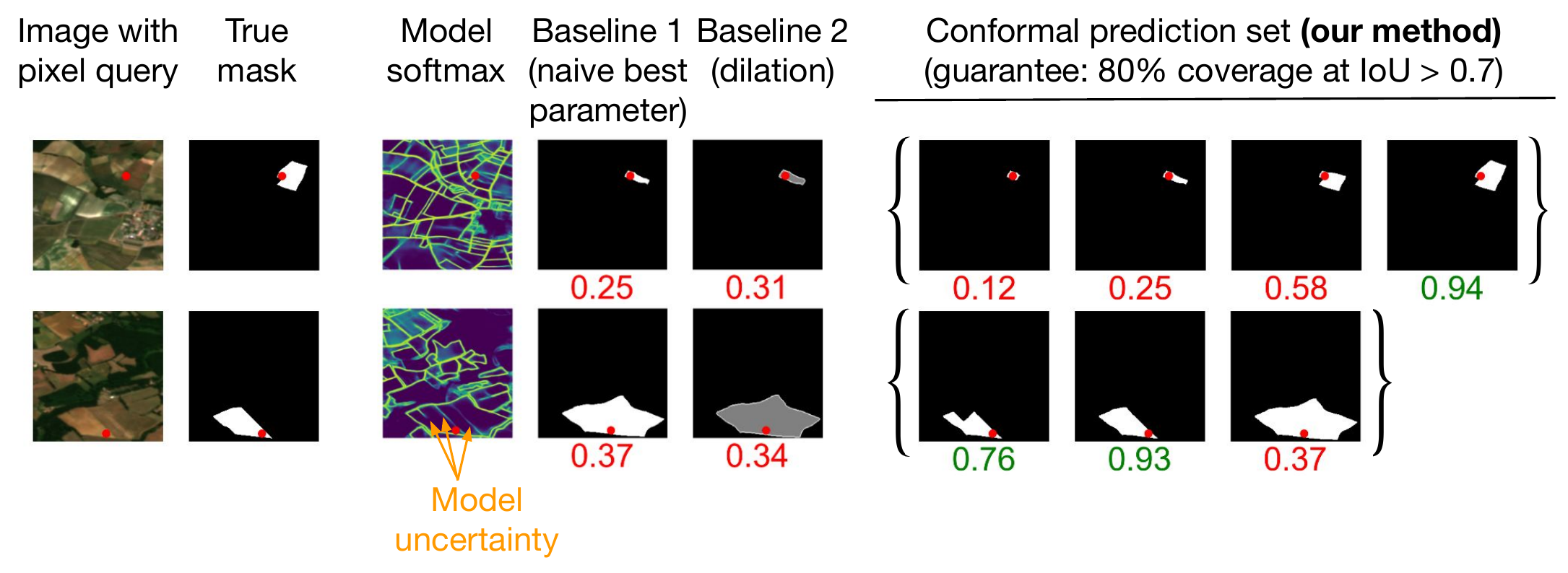}
  \caption{\textbf{Example field instance segmentation queries with true masks, model softmax scores, baseline predictions, and our method's conformal prediction sets (with IoU scores shown below each prediction).} Given an image and a pixel coordinate query, our conformal algorithm generates a confidence set of instance predictions for that pixel, with a provable guarantee for the probability that at least one of the predictions has high IoU (shown in green) with the true object instance mask. The naive best parameter baseline uses the single best model parameter value (over the calibration set) to generate a single prediction, but this often results in low IoU, as in the examples of over- and under-segmentation shown above. The dilation-based conformal baseline, which dilates the single prediction by a fixed number of pixels (determined using the calibration set), also fails to capture this structural ambiguity. By contrast, our method's confidence sets provide diverse predictions and adapt to query difficulty.}\label{fig:summary}
\end{figure*}

Instance segmentation aims to detect and delineate individual object instances in an image, producing a binary mask for each object. Despite achieving high performance on average predictions, current instance segmentation models lack principled uncertainty quantification: their outputs (such as logits or confidence scores) are not statistically calibrated, and there is no guarantee that a predicted mask is close to the ground truth. For example, the widely-used Segment Anything Model \cite{kirillov2023segment} outputs three binary masks for each query, along with their predicted Intersection-Over-Union (IoU) with the true mask, but there is no guarantee that any of these masks are correct or that their predicted IoUs are accurate. 
Similarly, ad-hoc uncertainty metrics, such as instance boundary probability predictions in agricultural field delineation \cite{waldner2021detect}, are not calibrated and may be difficult to interpret. 

Principled methods using conformal prediction have recently been created for segmentation, but remain constrained to modifications of a single model-predicted mask, e.g., through dilation or boundary expansion \cite{davenport2024conformal, mossina2025conformal}. These approaches capture local boundary uncertainty but cannot represent the structural ambiguity that dominates in many real applications, such as whether adjacent regions should be treated as a single object or split into multiple objects (see Appendix \ref{appendix:structural-uncertainty-metrics} for metrics quantifying this type of structural uncertainty). Addressing this gap requires confidence sets that contain qualitatively different masks, not just small perturbations of one mask.

To address this limitation, we introduce a conformal prediction algorithm to generate adaptive confidence sets for instance segmentation. Given an image and a pixel coordinate query, our algorithm generates a confidence set of instance predictions for that pixel, with a provable guarantee for the probability that at least one of the predictions has high IoU with the true object instance mask (Figure \ref{fig:summary}). Our algorithm (Section \ref{sec:method}) generates a diverse set of predictions for each query by varying a tunable parameter of an instance segmentation model, and uses a randomly sampled calibration dataset to find a set of parameter values that is likely to result in high IoU predictions. Furthermore, we remove duplicate predictions to construct prediction sets that vary in size based on query difficulty. We provide versions of the algorithm with asymptotic and finite sample guarantees. We apply our algorithm to instance segmentation examples (Section \ref{sec:experiments}) in agricultural field delineation, cell segmentation, and vehicle detection. 

The main contribution of our work is a novel formulation of conformal prediction for instance segmentation that captures structural uncertainty. Specifically, (1) we identify structural uncertainty not captured by existing conformal approaches based on boundary perturbations; (2) we introduce a conformal prediction formulation that targets this setting by constructing set-valued predictions over multiple model configurations, enabling coverage guarantees even when no single parameter value can achieve the desired IoU threshold; and (3) we provide a practical instantiation by using a parameter sweep and set-cover-based selection, which allows us to probe structural diversity in existing models.


\subsection{Related Work}
\paragraph{Conformal Prediction} Conformal prediction methods generate confidence sets for machine learning model predictions, with provable coverage guarantees. Given an ML model and a labeled calibration dataset of randomly sampled points, conformal prediction uses a heuristic metric to measure uncertainty in model predictions and predict confidence sets for future unlabeled data points sampled from the same distribution \cite{angelopoulos2023conformal}. 

Two broad classes of conformal methods that have been applied to image segmentation are Conformal Risk Control (CRC) \cite{angelopoulos2022conformal, blot2024automatically, luo2025conditional, luo2026enhancing} and Learn Then Test (LTT) \cite{angelopoulos2025learn}. \citet{angelopoulos2022conformal} introduced CRC and applied it to choose a model probability threshold for binary segmentation such that the expected false negative rate falls below a user-specified target. \citet{angelopoulos2025learn} introduced LTT and applied it to choose instance segmentation model parameters with guarantees for mean IoU, recall, and coverage.  

Like our paper’s method, both LTT and CRC are conformal methods that generate model predictions by varying a tunable parameter. However, these methods output a single prediction for each input, while our method outputs adaptive confidence sets with multiple predictions. 
The disadvantage of LTT and CRC is that there does not always exist a single parameter value that results in sufficiently high coverage on the calibration dataset. In our instance segmentation setting, there does not always exist a single parameter value that results in mask predictions with high IoU for a user-specified fraction $(1-\alpha)$ of calibration points. Thus, when attempting to guarantee high coverage on the test set, LTT or CRC can often return an error. 
In contrast, our method is feasible in more scenarios: it searches for a confidence set of multiple parameters such that with probability $(1-\alpha)$, predictions from \textit{at least one} of the parameters will have high IoU.

CRC has the additional disadvantage of requiring a loss function that is monotone or near-monotone in the tunable parameter. Similarly, Risk-Controlling Prediction Sets \cite{bates2021distribution}, a precursor to LTT, also requires monotone loss. However, IoU loss is not monotone or near-monotone. In contrast, our method does not require a monotone loss function, which allows us to use an IoU-based loss function. 

The general LTT framework described by \citet{angelopoulos2025learn} could potentially be adapted to our setting, but this requires greatly increasing the search space of tunable parameters. In our setting, if the tunable parameters are $t_1,...,t_k$, then the LTT framework could be considered where the tunable parameters are elements of the power set $2^{\{t_1,...,t_k\}}$ (which has $2^k$ elements) rather than individual $t_j$ values. However, this setting with an extremely large parameter space is not explored in the LTT paper, whereas we provide a prescriptive algorithm and implementation for producing a confidence set of multiple predictions.

We emphasize that our method extends the line of work introduced by the LTT and CRC frameworks: while prior instantiations of LTT/CRC select a single model configuration, we select a set of configurations whose predictions jointly achieve the desired coverage. Our method is the first to construct confidence sets of diverse segmentation predictions. Making the design choices that enable sets of predictions is non-trivial; we believe the construction of confidence sets rather than single predictions provides novelty and is necessary in order to capture structural uncertainty.

\paragraph{Uncertainty quantification for semantic and instance segmentation} Several recent works use conformal methods to generate prediction sets for image segmentation tasks. For semantic segmentation, \citet{mossina2024conformal} used softmax scores to construct pixel-level confidence sets of class predictions, and \citet{viti2025consign} developed image-level prediction sets that account for spatial correlations. These semantic segmentation methods do not distinguish between different object instances of the same class, and thus are not directly comparable to our method, which predicts a set of masks for each object instance. 

For binary instance segmentation, \citet{davenport2024conformal} combined predicted logit scores with boundary distances to construct inner and outer confidence sets, while \citet{mossina2025conformal} applied morphological dilation to expand a predicted binary mask into a margin-shaped outer confidence set. However, these approaches only guarantee that the true mask is contained within the outer confidence set (and/or contains the inner confidence set) with high probability, and do not guarantee high IoU between the true mask and either confidence set. Furthermore, the confidence set is constrained to be a modification of an existing prediction, such as a dilation or boundary expansion. By contrast, our method constructs sets of alternative instance masks, where members of the set may differ substantially in shape. This allows us to capture more diverse and realistic forms of uncertainty. While our method's predictions are constrained to those generated by varying the tunable parameter, these are substantially more diverse (see Figure \ref{fig:predictions-varying-t}) compared to, e.g., dilating a single predicted mask by a fixed margin.   

\paragraph{Uncertainty quantification for object detection} Conformal methods have also been applied to object detection tasks, allowing users to control the false negative rate \cite{li2022towards, andeol2023confident, andeol2025conformal} or number of false positives \cite{fisch2022conformal} in object bounding box predictions. Another approach is to use conformal prediction to construct confidence intervals for bounding box coordinates \cite{de2022object, timans2024adaptive, mukama2024copula, zouzou2025robust}. Bounding box predictions have limited flexibility and precision compared to mask predictions, which capture a wider range of object shapes. 

\section{METHOD}\label{sec:method}

\subsection{Setting} 

We use conformal prediction to generate adaptive confidence sets for instance segmentation tasks.  We consider the following setting:
\begin{itemize}
    \item The input $X=(I, (z_1, z_2))$ consists of an image $I \in \mathbb{R}^{W \times H \times C}$ 
    and a ``query point'' at pixel coordinate $(z_1, z_2)$.
    \item Each query is located in a ground truth object instance, which can be denoted as a binary mask $Y \in \{0,1\}^{W \times H}$.
    \item We have access to $f$, an instance segmentation model that takes as input $X=(I, (z_1, z_2))$ and a tunable parameter $T$ to output binary mask predictions $\hat{Y} = f(X, T)$. By varying $T$, we can modulate the behavior of $f$ and obtain a diverse set of predictions. Examples of $T$ include mask scores in Segment Anything, logit thresholds before connected-components analysis, and thresholds in watershed segmentation.
    \item To evaluate the quality of a mask prediction, we use the Intersection-over-Union metric,
    \begin{equation*}
        \text{IoU}(Y, \hat{Y}) = \frac{|Y \cap \hat{Y}|}{|Y\cup\hat{Y}|},
    \end{equation*} 
    where a perfect prediction has $\text{IoU}=1$. (For matrices $A,B \in \{0,1\}^{W \times H}$ we use $|A \cap B|$ and $|A \cup B|$ to denote the number of nonzero entries in the pointwise product $A \odot B$ and sum $A+B$, respectively). 
\end{itemize}
This setting arises in applications ranging from farmers selecting their fields in satellite imagery to modern interactive models like Segment Anything, which take point prompts and return object masks. While we describe the method in the point-query setting, it extends to producing prediction sets for an entire image --- either directly for models that support image-wide inference, or by querying all pixels in a brute-force manner. We adopt IoU as the evaluation metric for concreteness, though the method extends readily to other measures of segmentation quality.


Finally, to construct conformal prediction sets we require a calibration dataset $(X_1, Y_1), \dots, (X_n, Y_n)$ consisting of $n$ image-query pairs with their corresponding ground-truth masks. For a new test (or target) example $(X^{\text{test}}, Y^{\text{test}})$ coming from the same distribution, and given a user-specified IoU target $0 < \tau < 1$ and error rate $0 < \alpha < 1$, our conformal algorithm produces a confidence set $C_{\alpha, \tau}(X^{\text{test}})$. This confidence set is designed such that with probability $\ge 1-\alpha$, it contains at least one predicted mask $\hat{y} \in C_{\alpha, \tau}(X^{\text{test}})$ such that $\text{IoU}(Y^{\text{test}}, \hat{y}) > \tau$. For formal theoretical guarantees, we assume the calibration and test samples are IID:

\begin{assumption}\label{assump:IID}   $ (X_1,Y_1), \dots ,(X_n,Y_n) \stackrel{\text{iid}}{\sim} \mathbb{P}$, and independently, $(\xt,\yt) \sim \mathbb{P}$ for all $n \in \mathbb{Z}_+$.
\end{assumption}


\subsection{Conformal Instance Segmentation Algorithm}\label{sec:method-conformal-algorithm}
Our conformal algorithm proceeds as follows (Algorithm \ref{algorithm-conformal}). We begin by selecting a grid of $k$ possible values for the tunable parameter $T$, denoted $\{t_1, t_2, \dots, t_k\}$. Recall that $T$ is any parameter that controls how $f$ produces masks. Crucially, no single value of $T$ yields the best prediction for all queries: at some settings, the model succeeds on certain queries, while other settings perform well on different queries (Figure \ref{fig:predictions-varying-t}). Our conformal algorithm leverages this variability by selecting a subset of $T$ values that, together, provide coverage across the dataset.

\begin{figure}
  \centering
  \includegraphics[width=0.5\textwidth]{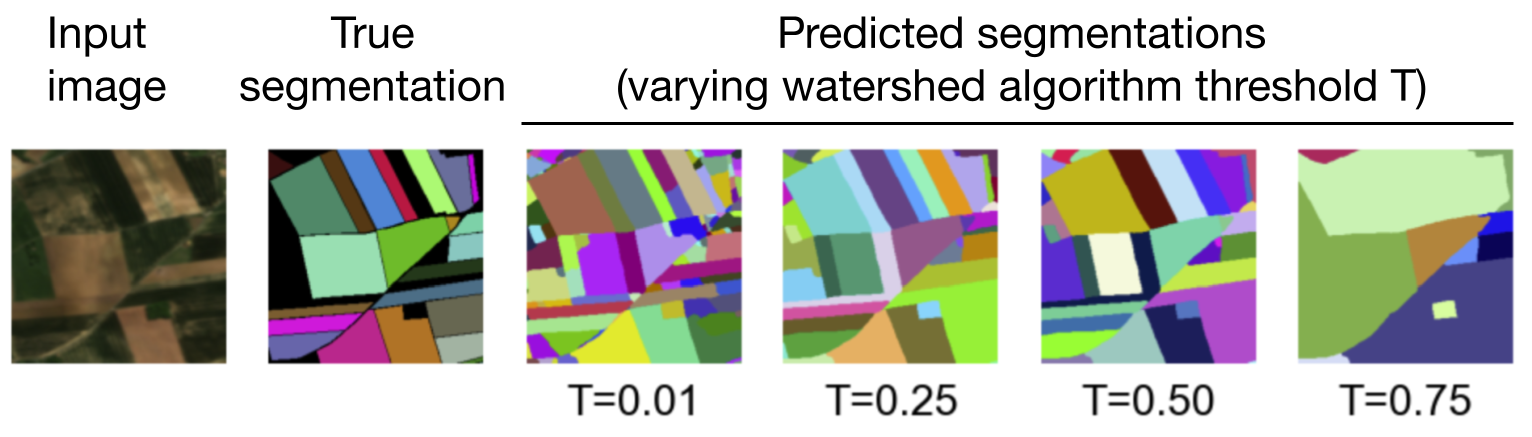}
  \caption{\textbf{Generating diverse predictions by varying tunable parameter $T$.} In the field delineation example, we vary the watershed algorithm threshold $T$ to generate multiple segmentations for each image. While a given threshold may oversegment or undersegment a specific field, another threshold often succeeds, motivating the use of prediction sets that combine them.}\label{fig:predictions-varying-t}
\end{figure}

\begin{algorithm}[!hbt]
\footnotesize
\caption{Conformal prediction set for instance segmentation}\label{algorithm-conformal}
\begin{algorithmic}
\State \textbf{Input:} 
\begin{itemize}
    \item calibration data $(X_1,Y_1),\dots,(X_n, Y_n)$
    \item test input $X^\text{test}$
    \item instance segmentation model $f(X, T)$
    \item parameter values $t_1,\dots,t_k$
    \item error rate $0 < \alpha < 1$
    \item target IoU $0 < \tau < 1$
\end{itemize}
\State \textbf{Output:} 
\begin{itemize}
    \item parameter indices $J_{\alpha, \tau}$ for prediction set
    \item conformal prediction set $C_{\alpha, \tau}(X^\text{test}) = \{ \hat{Y}_j^\text{test} : j \in J_{\alpha, \tau}\}$ 
\end{itemize}
\State
\end{algorithmic}
\begin{algorithmic}[1]
\For{$j=1,2,\ldots,k$}
    \For{$i=1,2,\ldots,n$}
        \State $\hat{Y}_{ij} \gets f(X_i, t_j)$
        \State $\rho_{ij} \gets \text{IoU}(Y_i, \hat{Y}_{ij})$
    \EndFor
    \State $S_j \gets \{i : \rho_{ij} > \tau \}$
\EndFor
\State $J_{\alpha, \tau} \gets \argmin_{J \subseteq [k]} \bigl\{ |J| : \big| \bigcup_{j \in J} S_j \big| \geq (1-\alpha)n \bigr\}$\label{line:find_minimal_set_cover}
\State $C_{\alpha, \tau}(X^\text{test}) \gets \{ f(X^\text{test}, t_j) : j \in J_{\alpha, \tau}\}$
\State \Return $C_{\alpha, \tau}(X^\text{test}), J_{\alpha, \tau}$
\end{algorithmic}
\end{algorithm}

For each calibration example $(X_i, Y_i)$ and each value $t_j$ (for $1 \leq i \leq n$ and $1 \leq j \leq k$), we use the instance segmentation model $f$ to predict a binary mask $\hat{Y}_{ij} = f(X_i, t_j)$
and evaluate its quality using the IoU with the true mask: $\rho_{ij} = \text{IoU}(Y_i, \hat{Y}_{ij})$.
Next, for each $t_j$, we collect the indices of calibration points where the prediction is sufficiently accurate:
\begin{align*}
    S_j = \{i  :  \rho_{ij} > \tau \} = \bigl\{ i : \iou \big(Y_i, f(X_i,t_j) \big) > \tau \bigr\} .
\end{align*}
In other words, $S_j$ is the set of calibration examples where $f(\cdot, t_j)$ achieves IoU above the threshold $\tau$. 

We then identify a minimal subset $J_{\alpha, \tau} \subseteq [k]$ of parameter values such that the union $\bigcup_{j \in J_{\alpha, \tau}} S_j$ covers at least $(1-\alpha) n$ calibration examples. Intuitively, this ensures that the selected parameter values collectively succeed on a large fraction of the calibration set. 

Finally, given a test image $X^\text{test}$, we generate predictions $\hat{Y}_j^\text{test} = f(X^\text{test}, t_j)$ for each $j \in J_{\alpha, \tau}$ and output the conformal prediction set
\begin{align*}
    C_{\alpha, \tau}(X^\text{test}) = \{ \hat{Y}_j^\text{test} : j \in J_{\alpha, \tau}\}.
\end{align*}

\paragraph{Computing a minimal-size set cover} The step of identifying $J_{\alpha, \tau}$ is a variant of the classical Set Cover problem, which is NP-hard \cite{Karp1972}. To make this practical, we use a two-stage approach. 
First, we apply the greedy set cover algorithm \cite{johnson1973approximation, lovasz1975ratio, chvatal1979greedy}, which runs in polynomial time, to obtain an initial cover $J'_{\alpha, \tau} \subseteq [k]$ that covers at least $(1-\alpha) n$ calibration examples. 
We then attempt to improve this solution by brute-forcing over all subsets of $[k]$ of smaller size, ordered by cardinality, until we either find a true minimal cover or confirm that $J'_{\alpha, \tau}$ is already minimal. In practice, this refinement step is feasible because the greedy solution typically yields a small set. (However, if the greedy set cover $J'_{\alpha, \tau}$ is too large for the brute-force refinement step to be computationally tractable, we omit the refinement step.) 

\paragraph{Conformal guarantee} Because the test point is drawn from the same distribution as the calibration data, our procedure inherits a conformal coverage guarantee. In particular, with probability at least $(1-\alpha)$ asymptotically, there is a predicted mask $\hat{y}$ in the confidence set $C_{\alpha, \tau}(X^{\text{test}})$ whose IoU with the true mask $Y^{\text{test}}$ is greater than $\tau$ (see Proposition \ref{prop:NonadaptiveAlgValidity}). These guarantees rely on a feasible choice of $\alpha,\tau$. 

We note that our method differs from standard conformal prediction in that Algorithm \ref{algorithm-conformal} does not use a non-conformity score and the theory does not use exchangeability arguments. However, we still consider our method to fall under the conformal prediction framework because it uses a randomly sampled calibration set to generate statistically valid confidence sets for individual model predictions. 

\paragraph{Not every $\alpha, \tau$ are possible} The quality of the conformal guarantees depends fundamentally on the underlying segmentation model $f$. If more than $\alpha n$ calibration points have no parameter $t_j \in \{t_1, \cdots, t_k\}$ that yields IoU greater than $\tau$, then no subset $J_{\alpha, \tau}$ can cover the required $(1-\alpha)n$ calibration examples. In such cases, Algorithm \ref{algorithm-conformal} Line \ref{line:find_minimal_set_cover} fails to run. To address such issues, a user can either increase the error rate $\alpha$, lower the IoU target $\tau$, or consider a broader collection of segmentation models by considering a superset of the segmentation model parameters $\{t_1,\dots,t_k\}$.


\subsection{Adaptive Prediction Sets via Duplicate Removal}\label{sec:method:remove-duplicates}

For any test input $X^\text{test}$, Algorithm \ref{algorithm-conformal} outputs a confidence set $C_{\alpha, \tau}(X^\text{test})$ whose size is fixed at $|J_{\alpha, \tau}|$. In practice, however, many of the masks $\hat{Y}_j^\text{test}$ for $j \in J_{\alpha, \tau}$ may be identical or highly similar, depending on the underlying model. To avoid redundancy, we post-process $C_{\alpha, \tau}(X^\text{test})$ using Algorithm \ref{algorithm-remove-duplicates} to remove near-duplicate masks. This yields confidence sets whose sizes adapt to the particular input $X^\text{test}$, providing a more informative representation of uncertainty.

\begin{algorithm}[!hbt]
\footnotesize 
\caption{Compute conformal prediction set for instance segmentation, with duplicates removed}\label{algorithm-remove-duplicates}
\begin{algorithmic}
\State \textbf{Input:} 
\begin{itemize}
    \item same inputs as Algorithm \ref{algorithm-conformal}
    \item duplicate prediction IoU threshold $0 < \eta < 1$
\end{itemize}
\State \textbf{Output:} 
\begin{itemize}
    \item conformal prediction set with duplicates removed $C_{\alpha, \tau, \eta}(X^\text{test})$
    \item new IoU threshold $\tilde{\theta}$ for conformal guarantee
\end{itemize}
\State
\end{algorithmic}
\begin{algorithmic}[1]
\Procedure{Unique}{$C$, $\eta$} 
    \For{$C' \subseteq C$ in increasing cardinality}\label{line:UniqueForLoop}
    \If{for all $\hat{Y} \in C \setminus C'$ there exists $\hat{Y}' \in  C'$ \\ \hspace{27pt} such that $\text{IoU}(\hat{Y}, \hat{Y}') > \eta$}
        \State \Return $C'$
    \EndIf
\EndFor
\EndProcedure
\State
\State \textit{// remove duplicates from conformal prediction set}
\State $C_{\alpha, \tau}(X^\text{test}), J_{\alpha, \tau} \gets \Call{Algorithm-1}{}$ \label{line:CallFirstAlg}
\State $C_{\alpha, \tau, \eta}(X^\text{test}) \gets \Call{Unique}{C_{\alpha, \tau}(X^\text{test}), \eta}$
\State
\State \textit{// re-calibrate conformal guarantee IoU threshold}
\For{$i=1,2,\ldots,n$}
    \State $C_{\alpha, \tau}(X_i) \gets \{ f(X_i, t_j) : j \in J_{\alpha, \tau}\}$
    \State $C_{\alpha, \tau, \eta}(X_i) \gets \Call{Unique}{C_{\alpha, \tau}(X_i), \eta}$ \label{line:CallUnique}
    \State $s_i \gets \max_{\hat{Y_i} \in C_{\alpha, \tau, \eta}(X_i)} \text{IoU}(Y_i, \hat{Y}_i)$
\EndFor
\State $\tilde{\theta} \gets \text{Quantile}_{\alpha} \{s_i\}_{i=1}^n$
\State \Return $C_{\alpha, \tau, \eta}(X^\text{test}), \tilde{\theta}$
\end{algorithmic}
\end{algorithm} 

\paragraph{Adaptive prediction set} Formally, let $0 < \eta < 1$ be a user-specified IoU threshold at which two masks are considered duplicates. We define
\begin{equation*}
    C_{\alpha, \tau, \eta}(X^\text{test}) \subseteq C_{\alpha, \tau}(X^\text{test})
\end{equation*}
as a minimal-size subset such that for each mask prediction $\hat{Y} \in C_{\alpha, \tau}(X^\text{test}) \setminus C_{\alpha, \tau, \eta}(X^\text{test})$, there exists $\hat{Y}' \in  C_{\alpha, \tau, \eta}(X^\text{test})$ satisfying $\text{IoU}(\hat{Y}, \hat{Y}') > \eta$. In other words, the new set is a minimal subset such that every discarded mask is within IoU $\eta$ of one that is kept.

Computing this minimal subset is equivalent to the Dominating Set problem in graph theory, which is also NP-hard \cite{garey1979}. In our implementation, we iterate over all possible subsets $C_{\alpha, \tau, \eta}(X^\text{test}) \subseteq C_{\alpha, \tau}(X^\text{test})$ in order of increasing cardinality until we find one that satisfies the above condition. This brute-force search is tractable in our experiments because $|C_{\alpha, \tau}(X^\text{test})| = |J_{\alpha, \tau}|$ is small. For larger set sizes, exhaustive search would quickly become infeasible, and one would need to resort to approximation or heuristic algorithms developed for the Dominating Set problem.

\paragraph{New conformal guarantee} After removing duplicates, the new prediction set $C_{\alpha, \tau, \eta}(X^\text{test})$ will not necessarily satisfy the original conformal guarantee. To recover a valid guarantee, we re-calibrate the IoU threshold. For each calibration point $(X_i, Y_i)$, we construct the adaptive set $C_{\alpha, \tau, \eta}(X_i)$ and record the best IoU achieved within that set:
\begin{align*}
    s_i = \max_{\hat{Y_i} \in C_{\alpha, \tau, \eta}(X_i)} \text{IoU}(Y_i, \hat{Y}_i).
\end{align*}
We then set $\tilde{\theta}$ to be the $\alpha$-quantile of the scores $\{s_i : 1 \leq i \leq n\}$ across the calibration dataset. The resulting guarantee is that, for a new test input $X^\text{test}$, with probability at least $1-\alpha$, there exists some $\hat{Y}^\text{test} \in C_{\alpha, \tau, \eta}(X^\text{test})$ with $\text{IoU}(Y^\text{test}, \hat{Y}^\text{test}) > \tilde{\theta}$. This is stated formally in the following theorem, which is proven in Section \ref{sec:ProofOfMainAlgortihmValidity}.

\begin{theorem}\label{theorem:ValidityMainTheorem}
 Let $C_{\alpha,\tau,\eta}^{(n)}(\xt)$ and $\tilde{\theta}^{(n)}$ denote the set of segmentation masks and $\iou$ threshold returned when running Algorithm \ref{algorithm-remove-duplicates} with calibration samples $\big( (X_i,Y_i) \big)_{i=1}^n$ and parameters $\alpha,\tau,\eta \in (0,1)$. Under Assumptions \ref{assump:IID}, \ref{assump:TieBreakingDeterministic}, and \ref{assump:StrictFeasibilitySmallestSet}, $$\liminf\limits_{n \to \infty} \mathbb{P} \Big(  \max_{\hat{y} \in C_{\alpha,\tau,\eta}^{(n)}(\xt) } \bigl\{ \iou( \yt,\hat{y}) \bigr\} \geq \tilde{\theta}^{(n)} \Big) \geq 1-\alpha.$$
\end{theorem}

Theorem \ref{theorem:ValidityMainTheorem} holds under the IID assumption (Assumption \ref{assump:IID}), and two technical assumptions introduced in Section \ref{sec:AdditionalAssumptions}. These additional assumptions will typically hold, provided that the user makes appropriate implementation and input parameter choices. 

While $\tilde{\theta}$ may differ from $\tau$, we do not expect it to differ substantially as long as the duplicate removal IoU threshold $\eta$ is high, since confidence sets that cover the true mask will often still cover the true mask after removing redundant predictions.

\paragraph{Finite sample guarantees} The above theorem provides asymptotic guarantees. In Appendix \ref{appendix:finite-sample-guarantees}, we present a modified version of our algorithm with finite sample guarantees. We randomly split the calibration data into two sets; we use the first set to rank parameter values $t_1, t_2, \dots t_k$ from highest to lowest priority via a greedy algorithm, then apply Conformal Risk Control to the second set to compute the smallest value $\hat{\lambda} \in [k]$ such that the first $\hat{\lambda}$ parameters in the ranked list result in a valid conformal prediction set. (The sample splitting allows for finite sample guarantees; in contrast, reusing the calibration data between Algorithms \ref{algorithm-conformal} and \ref{algorithm-remove-duplicates} breaks exchangeability and is a major reason we are only able to obtain asymptotic guarantees for Algorithm \ref{algorithm-remove-duplicates}.)

The finite-sample and asymptotic algorithms will generally give similar results because the greedy procedure used in the finite-sample algorithm is closely related to the greedy set cover algorithm used in Line \ref{line:find_minimal_set_cover} of Algorithm \ref{algorithm-conformal}. However, the asymptotic version's prediction sets may be smaller in some cases because in the asymptotic version (1) we can improve the greedy set cover by brute-forcing over all smaller subsets to find the minimal-size set cover, and (2) the duplicate removal procedure brute-forces over all subsets to produce a minimal-size output. In the finite-sample algorithm, because CRC requires a loss function that is non-decreasing in $\lambda$, the range of possible prediction sets is more restricted and we are unable to guarantee minimal-size outputs. We note that the finite-sample algorithm has the advantage of computational efficiency because its greedy parameter ranking procedure and duplicate removal procedure run in polynomial time, making it useful in cases where brute-forcing is not computationally tractable.

In our experiments in the following section, we use the asymptotic version of the algorithm. (For experimental results using the finite-sample version of the algorithm, see Appendix \ref{appendix:finite-sample-experiments}.)

\section{EXPERIMENTS}\label{sec:experiments}

We evaluate our algorithm on three major use cases of instance segmentation where uncertainty quantification is important: agricultural field delineation, cell segmentation, and vehicle detection. For each dataset, the inputs consist of images paired with pixel-wise instance segmentation labels. All images were held out from the training of the segmentation models, i.e., they come from the models' test sets. We further divide these images into two disjoint subsets: a calibration set used to calibrate our conformal procedure, and a test set used to evaluate coverage and accuracy. For each image, we randomly sample pixel coordinates to serve as queries, and vary the tunable parameter $T$ of the segmentation model to generate multiple instance predictions per query. We summarize the datasets, models, tunable parameters, and number of calibration and test pixels in Table \ref{table:datasets}. For detailed descriptions, see Appendix \ref{appendix:datasets}.

\renewcommand{\tabcolsep}{5pt}
\begin{table}[h]
\caption{\textbf{Datasets and models used in our experiments.}}\label{table:datasets}
\begin{center}
\tiny
\begin{tabular}{p{0.7cm}p{1.6cm}p{1.6cm}p{1.5cm}p{0.3cm}p{0.4cm}}
\toprule
EXAMPLE & DATASET & MODEL & TUNABLE \newline PARAMETER(S) & CALIB. \newline PIXELS & TEST \newline PIXELS \\ \midrule 
Field \newline delineation & Fields of The World \newline \cite{kerner2025fields} & FoTW UNet, \newline watershed algorithm & watershed threshold \newline $0 \leq T \leq 1$ & 2476 & 917 \\ \midrule
Cell \newline segmentation & Cellpose \newline \cite{stringer2021cellpose} & Cellpose-SAM \newline \cite{pachitariu2025cellpose} & cell extent threshold \newline $-5 \leq T \leq 5$ & 925 & 659 \\ \midrule
Vehicle \newline detection & Cityscapes \newline \cite{cordts2016cityscapes} & Segment Anything Model (SAM) \newline \cite{kirillov2023segment} & mask index \newline $T_1 \in \{1, 2, 3\}$, probability threshold \newline $0 \leq T_2 \leq 1$ & 105 & 121 \\ \bottomrule
\end{tabular}
\end{center}
\end{table}

Our procedure follows the framework described in Section \ref{sec:method}.
In each experiment, we calibrate our algorithm using the calibration queries, then construct $(1-\alpha)$ confidence sets of instance predictions for each test query. To evaluate our outputs, we visualize the prediction sets, check that they satisfy the $(1-\alpha)$ conformal coverage guarantee, and compare our results with baseline model predictions. 

The choice of error rate $\alpha$ and IoU threshold $\tau$ is task-specific, reflecting the quality of the underlying segmentation model; the values for each experiment are reported in Section \ref{sec:results}.
For duplicate removal, we use $\eta=0.9$ in all experiments; in practice, an end user would specify $\eta$ based on what level of overlap would be considered a duplicate in their application. (See Appendix \ref{appendix:sensitivity-analyses} for sensitivity analyses in which we vary $\alpha$, $\tau$, $\eta$, tunable parameter space size $k$, and calibration set size $n$.)

\subsection{Baselines}
\paragraph{Naive best parameter baseline} We create a ``naive best parameter baseline" that always outputs the single best tunable parameter value, i.e. the parameter value that results in the highest fraction of predictions with IoU $> \tau$ over the calibration set. This baseline represents the best possible coverage we can obtain using a single parameter value. For the field delineation, cell segmentation, and vehicle detection calibration datasets, the single best parameter values are respectively $T=0.24$, $T=0$, and $T=(1, 0.05)$. 
\paragraph{Morphological dilation baseline} We also compare our method to the morphological dilation-based method from \citet{mossina2025conformal}. This method constructs conformal confidence sets for binary segmentation using dilation, i.e., adding a margin to a predicted mask, where the margin size is a fixed number determined using calibration data. This algorithm only guarantees that the true mask is contained within the dilated mask with high probability, and does not guarantee high IoU. We find that the dilation-based method results in low coverage for IoU, undersegmentation, and lack of flexibility in predictions compared to our method. See Appendix \ref{appendix:dilation-baseline} for details on the experimental setup and results.

\subsection{Results}
\label{sec:results}

\begin{figure*}[]
  \centering
  \includegraphics[width=0.9\textwidth]{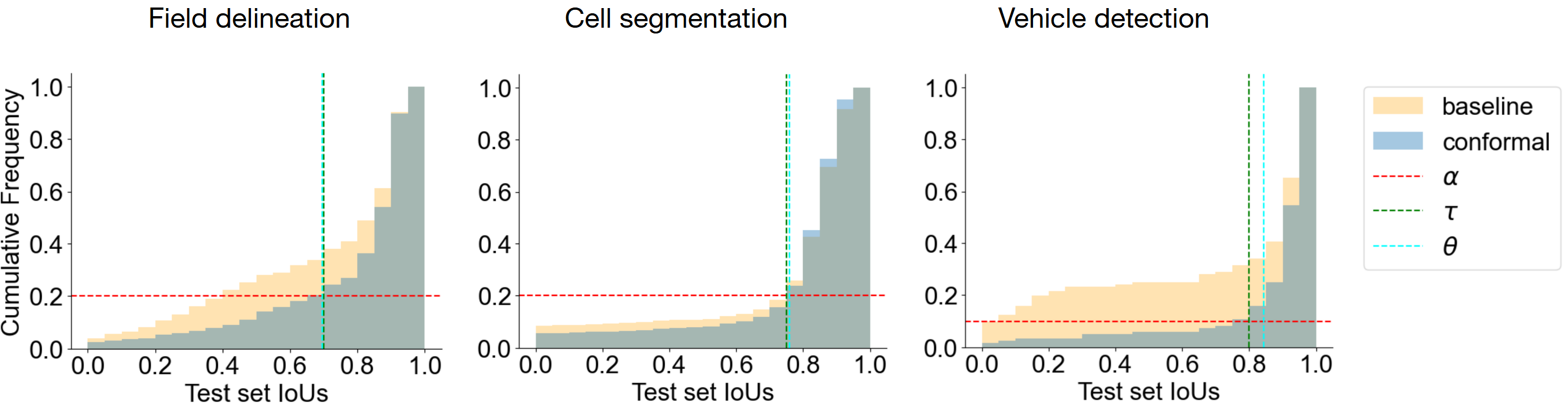}
  \caption{\textbf{Cumulative distributions of IoUs of naive best parameter baseline predictions and our conformal prediction sets over test data points.} For our conformal prediction sets, the maximum IoU over all predictions is used. In all examples, our conformal coverage for $\text{IoU} > \tilde{\theta}$ is close to target coverage $(1-\alpha)$ and greater than baseline coverage. Furthermore, the re-calibrated IoU threshold $\tilde{\theta}$ is close to (or greater than) the original threshold $\tau$, so removing duplicates does not significantly affect the original conformal guarantee. Note that the value of the baseline cumulative frequency at $\tau$ corresponds to the smallest error rate at which the naive best parameter baseline would provide coverage guarantees.}\label{fig:results-iou-histograms}
\end{figure*}

\begin{figure*}[]
  \centering
  \includegraphics[width=0.8\textwidth]{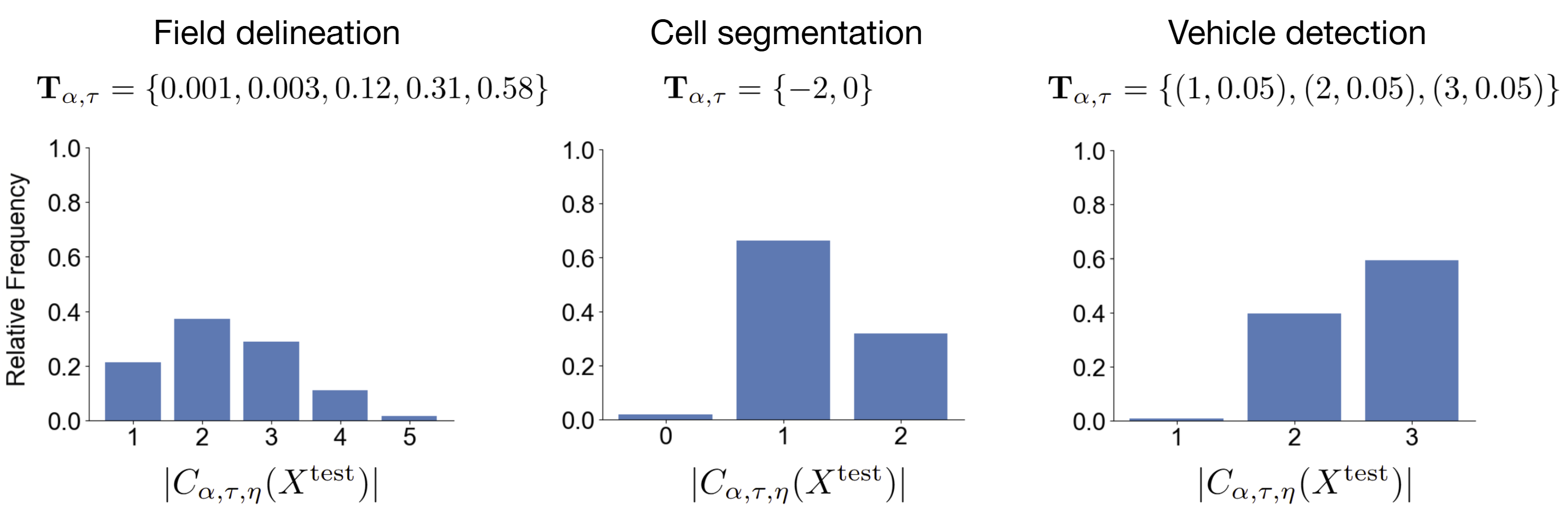}
  \caption{\textbf{Distribution of conformal prediction set sizes for test data points, after removing duplicate predictions.} Our algorithm constructs adaptive confidence sets that vary in size for different inputs. $\mathbf{T}_{\alpha, \tau}$ is the set of values of tunable parameter $T$ used to generate the initial prediction set before removing duplicates.}\label{fig:results-prediction-sets}
\end{figure*}

\begin{figure}[]
  \centering
  \includegraphics[width=0.5\textwidth]{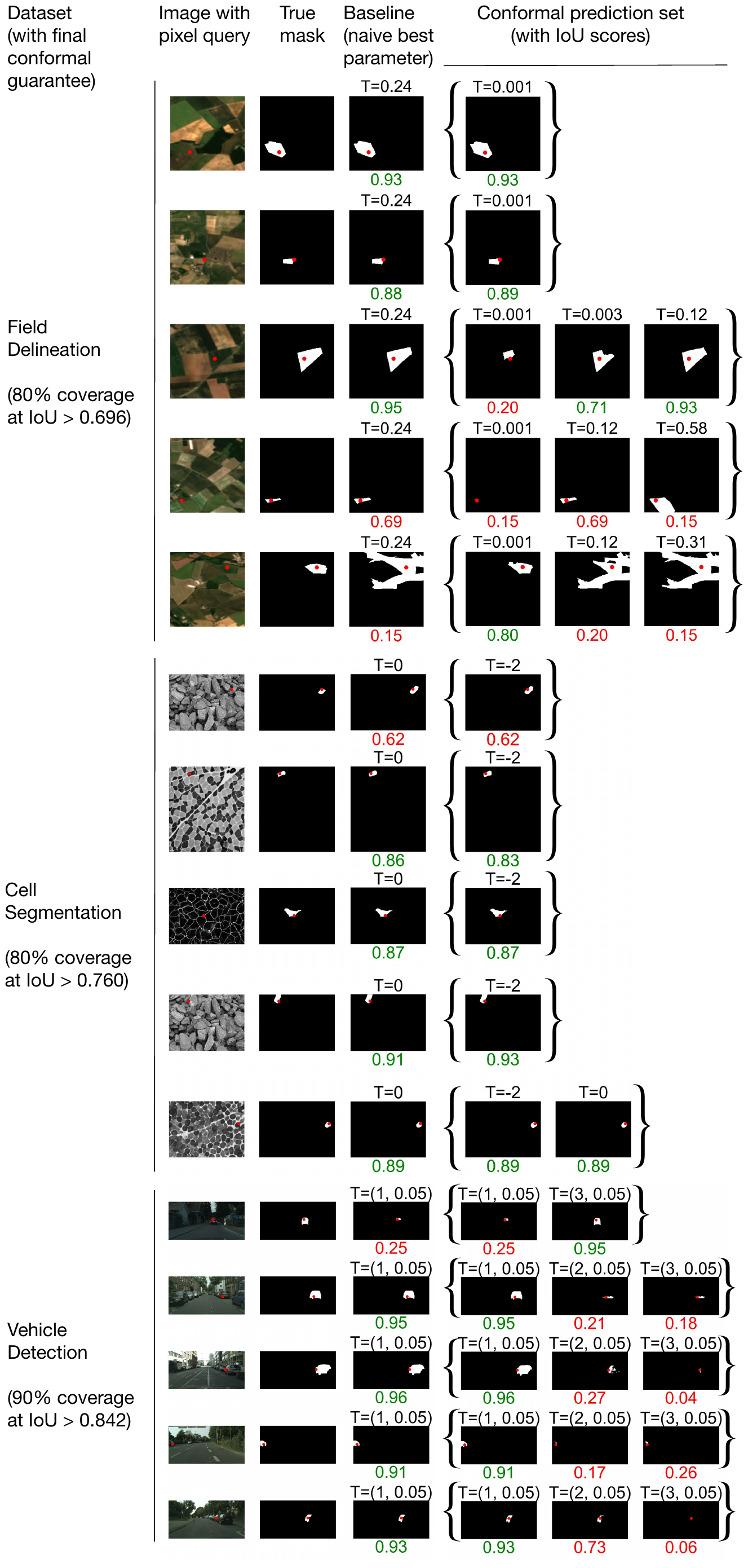}
  \caption{\textbf{Conformal prediction sets for example test queries, compared to naive best parameter baseline predictions.} Our conformal sets often include masks of different sizes or shapes, capturing structural uncertainty. Most contain at least one mask with high IoU (green) with the ground truth, and the number of predictions adapts to query difficulty. In contrast, the naive best parameter baseline only outputs a single prediction and does not guarantee high coverage.}\label{fig:results-images}
\end{figure}

\paragraph{What $\alpha, \tau$ are possible?} 
For each experiment, we minimized the target error rate $\alpha$ while maximizing IoU threshold $\tau$, subject to feasibility. To guide our choices, we plotted histograms of the maximum IoU achievable under any value of the tunable parameter $T$ for each experiment (Figure \ref{fig:calibration-iou-histograms}). This analysis revealed the inherent performance ceilings of the base models. For example, in the field delineation task, it was not possible to construct a 90\% confidence set for $\text{IoU} > 0.8$, because even when taking the best $T$ for each query, $>10\%$ of predictions had $\text{IoU} < 0.8$. (We note that choosing $\alpha$ and $\tau$ by inspecting calibration histograms is a form of double-dipping in the calibration data, albeit a mild one as we do not expect it to have a large effect on the empirical coverage in practice. To avoid this, a user could prespecify $\alpha$ and $\tau$; however, if the preselected $\alpha$ is too small for the preselected IoU threshold $\tau$, the algorithm would return an error message as it is unable to guarantee $(1-\alpha)$ coverage.)

This feasibility analysis defines a Pareto frontier of achievable $(\alpha,\tau)$ pairs: decreasing $\alpha$ (tighter error guarantees) necessarily requires lowering $\tau$, and conversely, raising $\tau$ requires allowing larger $\alpha$.
We use ($\alpha=0.2,\tau=0.7$) for fields, ($\alpha=0.2,\tau=0.75$) for cells, and ($\alpha=0.1,\tau=0.8$) for vehicles, choices that lie on the frontier and reflect the strictest feasible guarantees under the respective models (Table \ref{table:coverage}). 
Field delineation has the weakest guarantee, as it is the most ambiguous task among the three, while SAM vehicle segmentation has the strongest guarantee.
Quantifying these ceilings shows how far current models are from reliability levels needed in practice, and it clarifies that conformal prediction can certify but not overcome the limits of the underlying model.

\paragraph{Conformal prediction set coverage}
Given these feasible $(\alpha,\tau)$ pairs, our conformal algorithm achieves empirical coverage close to the target $1-\alpha$ and substantially higher than naive best parameter baselines (Table~\ref{table:coverage}). Figure~\ref{fig:results-iou-histograms} shows cumulative distributions of IoUs for naive best parameter baseline predictions versus conformal sets, where for each test query we take the maximum IoU among the masks in the set. In field delineation, our conformal sets increase coverage from 66.3\% to 79.7\%, recovering many cases where the baseline either merged multiple fields or split one field incorrectly. In vehicle detection, coverage improves from 66.1\% to 84.3\%, demonstrating that alternative hypotheses often capture cars missed by the top SAM prediction. By contrast, cell segmentation shows only a modest improvement (80.3\% to 83.5\%), reflecting that most errors in this domain are boundary-local. When visualizing example conformal prediction sets, we see that the sets usually contain at least one mask that is very similar to the true mask (Figure \ref{fig:results-images}). In general, using a confidence set of multiple parameters makes it more likely that at least one of the resulting predictions will have high IoU, compared to outputting a single prediction as the naive best parameter baseline does (see Appendix \ref{appendix:naive-best-parameter-baseline} for an illustrative example).

In all three examples, $\tilde{\theta} > \tau$ or $\tilde{\theta} \approx \tau$, so removing duplicate predictions does not significantly affect the original conformal guarantee. (Note that $\tilde{\theta} > \tau$ occurs when $\bigcup_{j \in J_{\alpha, \tau}} S_j$ contains more than $(1-\alpha) n$ calibration indices, so using the original IoU threshold $\tau$ would result in over-coverage.)

\renewcommand{\tabcolsep}{5pt}
\begin{table}[h]
\caption{\textbf{Target, Conformal, and Naive Best Parameter Baseline Coverages of Test Set Predictions.}}\label{table:coverage}
\begin{center}
\tiny
\begin{tabular}{p{1.3cm}p{0.3cm}p{0.3cm}p{1.7cm}p{0.8cm}p{1.8cm}}
\toprule
EXAMPLE & $\tau$ & $\tilde{\theta}$ & TARGET COVERAGE \newline $(1-\alpha)$ & CONFORMAL \newline COVERAGE & BASELINE COVERAGE (NAIVE BEST \newline PARAMETER) \\ \midrule 
Field delineation & 0.7 & 0.696 & 0.8 & 0.797 & 0.663 \\ \midrule
Cell segmentation & 0.75 & 0.760 & 0.8 & 0.835 & 0.803 \\ \midrule
Vehicle detection & 0.8 & 0.842 & 0.9 & 0.843 & 0.661 \\ \bottomrule
\end{tabular}
\end{center}
\end{table}

\paragraph{Adaptive set size}

In addition to achieving calibrated coverage, the experiments illustrate that our method adapts the size of prediction sets to query difficulty. Figure~\ref{fig:results-prediction-sets} shows the distribution of set sizes after duplicate removal: in most cases, sets contain three or fewer masks, but they expand when the query is ambiguous. This adaptivity improves efficiency, ensuring that users are not overwhelmed by redundant masks while still being presented with multiple plausible options in challenging cases. Field delineation has the largest set sizes, going up to 5 potential masks. The ambiguous field boundaries yield masks of significantly different sizes, corresponding to under- and over-segmentation hypotheses (Figure~\ref{fig:results-images}).  In contrast, for cells, most sets collapse to size one or two, again consistent with the observation that errors are primarily boundary-local. (A few sets have size zero because the model classifies the queried pixel as a non-cell.) In vehicle detection, our method often recovers useful predictions beyond the top-1 mask returned by SAM. Together, these results show that conformal sets provide not only calibrated coverage but also an adaptive representation of uncertainty that reflects the difficulty of each query. This adaptivity and diversity of predictions is not available in the naive best parameter baseline, which always outputs a single mask. 

\section{DISCUSSION}
We introduced a conformal prediction algorithm for instance segmentation that constructs adaptive sets of candidate masks with provable coverage guarantees. 
These guarantees formalize the reliability of instance segmentation models by making explicit the probability that at least one mask in the set achieves a desired IoU with the ground truth.
In both theory and practice, these sets attain the desired coverage and adapt their size to query difficulty.

Conformal prediction can only perform as well as the underlying segmentation models allow. In our experiments, empirical coverage is capped by the accuracy of the base models, which qualitatively fall short of human-level performance. This gap reflects opportunities for improvement through better algorithms and training strategies. At the same time, some degree of uncertainty is unavoidable, since the imagery itself can be ambiguous --- e.g., due to low contrast between adjacent fields in satellite images, overlapping cells in microscopy, or blurry vehicles in street scenes. In such cases, even a perfect model should not be expected to commit to a single mask. There is therefore value in prediction diversity, i.e., representing multiple plausible objects when the imagery is ambiguous. 
Current segmentation models are optimized to minimize segmentation loss, which does not explicitly incentivize producing multiple plausible alternatives when the input is ambiguous. We leave the development of training strategies that promote greater mask diversity to future work.


Several limitations remain. First, the conformal guarantee is global: every input receives the same IoU threshold, even though some queries are easy while others are ambiguous. Second, our algorithm requires varying a tunable parameter. In settings where such a parameter is not available or does not induce meaningful diversity, constructing conformal sets may be challenging. Lastly, the set construction steps (set cover, duplicate removal) involve combinatorial optimization. Although feasible in our experiments due to small set sizes, these steps may become computationally expensive in some settings. Developing scalable approximations could expand the practical applicability of the method.


\begin{acknowledgements} 
    This material is based upon work supported by the U.S. Department of Energy, Office of Science, Office of Advanced Scientific Computing Research, Department of Energy Computational Science Graduate Fellowship under Award Number DE-SC0023112. This report was prepared as an account of work sponsored by an agency of the United States Government. Neither the United States Government nor any agency thereof, nor any of their employees, makes any warranty, express or implied, or assumes any legal liability or responsibility for the accuracy, completeness, or usefulness of any information, apparatus, product, or process disclosed, or represents that its use would not infringe privately owned rights. Reference herein to any specific commercial product, process, or service by trade name, trademark, manufacturer, or otherwise does not necessarily constitute or imply its endorsement, recommendation, or favoring by the United States Government or any agency thereof. The views and opinions of authors expressed herein do not necessarily state or reflect those of the United States Government or any agency thereof. 
    
    DMK was partially supported by Generali through its research partnership with the Laboratory for Information and Decision Systems at MIT.
\end{acknowledgements}

\subsubsection*{References}
\printbibliography[heading=none]

\newpage

\onecolumn

\title{Supplementary Material}
\maketitle

\appendix
\renewcommand{\theassumption}{A.\arabic{assumption}}
\counterwithin{figure}{section}

\section{Code and data}
Code and data are available at \href{https://github.com/conformal-instance-segmentation/conformal-instance-segmentation-code}{https://github.com/conformal-instance-segmentation/conformal-instance-segmentation-code} 

\section{License information of datasets}
The Fields of the World dataset \cite{kerner2025fields} for France is under the French Open License (Licence Ouverte 2.0). This means the data must be attributed to Registre Parcellaire Graphique (RPG) and its provider, users must avoid implying endorsement, and users may freely use, modify, and redistribute the data for any purpose.

The Cellpose dataset \cite{stringer2021cellpose} is under the Creative Commons NonCommercial (CC-BY-NC) license. 

The Cityscapes dataset \cite{cordts2016cityscapes} is under a custom license: ``This Cityscapes Dataset is made freely available to academic and non-academic entities for non-commercial purposes such as academic research, teaching, scientific publications, or personal experimentation. Permission is granted to use the data given that you agree to our license terms." \href{https://www.cityscapes-dataset.com/license}{https://www.cityscapes-dataset.com/license}

\section{Proof of Validity of Algorithms}
\subsection{Additional Assumptions Needed}\label{sec:AdditionalAssumptions}

Recall Assumption \ref{assump:IID} that for each $n$, the calibration samples $(X_1,Y_1),\dots, (X_n,Y_n) \stackrel{iid}{\sim} \mathbb{P}$ and independently, $(\xt,\yt) \sim \mathbb{P}$. Throughout this supplement, we let $(X,Y)$ denote an arbitrary random object from the distribution $\mathbb{P}$.

The next assumption ensures that $1-\alpha$ and $\tau$ are small enough to be feasible.

\begin{assumption}\label{assump:FeasibleThresh}
    $\mathbb{P} \Big( \max_{j \in [k]} \bigl\{ \iou \big(Y,f(X,t_j) \big) \bigr\} > \tau \Big) > 1-\alpha.$
\end{assumption}

Assumption \ref{assump:FeasibleThresh} can be checked empirically, and assuming the calibration sample size $n$ is relatively large, Algorithm \ref{algorithm-conformal} will return an error when it does not hold. If Assumption \ref{assump:FeasibleThresh} does not hold, the user can either increase the size of the parameter search space $\{t_1,\dots,t_k\}$, can increase $\alpha$, or can decrease $\tau$. They can also consider some combination of these options. Algorithm \ref{algorithm-conformal} can subsequently be applied after the parameters are changed such that it will not return an error.

The next assumption is an implementation criterion that the user can enforce. We suspect that the assumption will also hold in most default implementations of Algorithms \ref{algorithm-conformal} and \ref{algorithm-remove-duplicates}. 

\begin{assumption}\label{assump:TieBreakingDeterministic}
    In Algorithms \ref{algorithm-conformal} and \ref{algorithm-remove-duplicates}, whenever a search across sets of the same cardinality is conducted, the ordering of sets with the same cardinality is fixed. In particular, in
    Line \ref{line:find_minimal_set_cover} of Algorithm \ref{algorithm-conformal}, for each $l \in [k]$, there is a predetermined ordering of all subsets $J \subseteq [k]$ with $\vert J \vert=l$. Thus in cases where the $\argmin$ in Line \ref{line:find_minimal_set_cover} of Algorithm \ref{algorithm-conformal} is equal to a collection $\{J_1,\dots,J_r\}$ of distinct subsets of $[k]$ with cardinality $l$, $J_{\alpha,\tau}$ will be set to the first element of $\{J_1,\dots,J_r\}$ in the predetermined ordering. Moreover, in Line \ref{line:UniqueForLoop} of Algorithm \ref{algorithm-remove-duplicates}, the for loop always uses the same ordering to search through sets with the same cardinality.
\end{assumption}

We introduce one more technical assumption that ensures that asymptotically, $J_{\alpha,\tau}$ returned by Algorithm \ref{algorithm-conformal} will equal $J_{\alpha,\tau}^*$ with probability converging to 1. First we must introduce some additional notation. For each $J \subseteq [k]$ define \begin{equation}\label{eq:MuJDef}
    \mu_{\tau}(J) \equiv \mathbb{P}\Big( \max_{j \in J} \bigl\{ \iou \big(Y,f(X,t_j) \big) \bigr\} > \tau \Big).
\end{equation} Further define $$\ca_{\alpha,\tau} \equiv \bigl\{ J \subseteq [k] \ : \  \mu_{\tau}(J)  \geq 1-\alpha \bigr\},$$ and observe that, by Assumption \ref{assump:FeasibleThresh}, $\ca_{\alpha,\tau}$ is nonempty. Now let $J_{\alpha,\tau}^*$ be an element of $\argmin_{J \subseteq \ca_{\alpha,\tau}} \{ \vert J \vert \}$. Under Assumption \ref{assump:TieBreakingDeterministic}, we can further define $ J_{\alpha,\tau}^*$ to be the element of $\ca_{\alpha,\tau}$ that is first among all elements in $\ca_{\alpha,\tau}$ with cardinality $\vert J_{\alpha,\tau}^* \vert$ according to the ordering used in Line \ref{line:find_minimal_set_cover} of Algorithm \ref{algorithm-conformal}. While by definition, $\mu_{\tau}(J_{\alpha,\tau}^*) \geq 1-\alpha$, the next assumption states that this inequality is strict, which should generally hold (for example, if there were a continuous prior on all $\mu_{\tau}(J)$ values, this next assumption would hold with probability 1 under Assumption \ref{assump:FeasibleThresh}).

\begin{assumption}\label{assump:StrictFeasibilitySmallestSet} $\ca_{\alpha,\tau}$ is nonempty and $\mu_{\tau}(J_{\alpha,\tau}^*) > 1-\alpha$.
\end{assumption}

\subsection{Helpful Lemma}

To prove the asymptotic validity of Algorithms \ref{algorithm-conformal} and \ref{algorithm-remove-duplicates}, it helps to introduce a lemma establishing the asymptotic behavior of the set $J_{\alpha,\tau}$ returned in Algorithm \ref{algorithm-conformal}. To this it is convenient to define, $$\ca_{\alpha,\tau}^c \equiv \bigl\{ J \subseteq [k] \ : \  \mu_{\tau}(J)  < 1-\alpha \bigr\}.$$

\begin{lemma}\label{lemma:ProbToZero}
    Let $J_{\alpha,\tau}^{(n)}$ denote the set of indices returned when running Algorithm \ref{algorithm-conformal} with calibration samples $\big( (X_i,Y_i) \big)_{i=1}^n$ and parameters $\alpha,\tau \in (0,1)$, and suppose $J_{\alpha,\tau}^{(n)}=\textnormal{NA}$ if Algorithm \ref{algorithm-conformal} returns an error. Then under Assumptions \ref{assump:IID} and \ref{assump:FeasibleThresh}, $\lim_{n \to \infty} \mathbb{P} (J_{\alpha,\tau}^{(n)} \in \ca_{\alpha,\tau}^c)=0$ and $\lim_{n \to \infty} \mathbb{P} (J_{\alpha,\tau}^{(n)} =\textnormal{NA})=0$.
\end{lemma}

\begin{proof}
    Fix $\alpha, \tau \in (0,1)$ to be the tolerance and $\iou$ thresholds used in Algorithm \ref{algorithm-conformal}. Let $\cx =\mathbb{R}^{W \times H \times C} \times [W] \times [H]$ and $\cy =\{0,1\}^{W \times H}$ denote the input and output space for $f(\cdot,t)$, respectively. Now for each $J \subseteq [k]$, define $g_J : \cx \times \cy \to \{0,1\}$ by $$g_J(x,y) \equiv \begin{cases}
    1 & \text{if } \iou \big(y , f(x,t_j) \big) > \tau \text{ for some } j \in J, 
    \\ 0 & \text{otherwise.} \end{cases}$$ Observe that for all $J \subseteq [k]$, $$\e[g_J(X,Y)]= \mathbb{P} \Big( \max_{j \in J} \bigl\{ \iou \big(Y,f(X,t_j) \big) \bigr\} > \tau \Big) = \mu_{\tau}(J),$$ where the last step follows by definition of $\mu_{\tau}(J)$ in Equation \eqref{eq:MuJDef}.

     Now fix $J \in \ca_{\alpha,\tau}^c$, and we will show that  $\lim_{n \to \infty} \mathbb{P}(J_{\alpha,\tau}^{(n)} =J)=0$. Since, $J \in \ca_{\alpha,\tau}^c$, $ \mathbb{E}[g_J(X,Y)] = \mu_{\tau}(J) < 1-\alpha$. Further by by Assumption \ref{assump:IID}, $g_J(X_i,Y_i) \in \{0,1\}$ are IID for $i=1,\dots,n$, so by the weak law of large numbers, \begin{equation}\label{eq:ConvergenceInProbJinAcomp}
\frac{1}{n} \sum_{i=1}^n g_J(X_i,Y_i) \xrightarrow{p} \mu_{\tau}(J)  < 1 - \alpha \quad \text{as} \quad n \to \infty. \end{equation}  Further observe that by Algorithm \ref{algorithm-conformal}, if the returned set of indices $J_{\alpha,\tau}^{(n)}$ equals $J$, it must be the case that $\sum_{i=1}^n g_J(X_i,Y_i) \geq (1-\alpha)n$. Hence $$\mathbb{P}(J_{\alpha,\tau}^{(n)} =J) \leq \mathbb{P} \Big( \sum_{i=1}^n g_J(X_i,Y_i) \geq (1-\alpha)n \Big)  \leq \mathbb{P} \Big( \Big| \frac{1}{n} \sum_{i=1}^n g_J(X_i,Y_i) - \mu_{\tau}(J) \Big| \geq (1-\alpha) -\mu_{\tau}(J) \Big).$$

Taking the limit as $n \to \infty$ of each side, by \eqref{eq:ConvergenceInProbJinAcomp} and the definition of convergence in probability, $$\limsup\limits_{n \to \infty} \mathbb{P}(J_{\alpha,\tau}^{(n)} =J) \leq  \limsup\limits_{n \to \infty} \mathbb{P} \Big( \Big| \frac{1}{n} \sum_{i=1}^n g_J(X_i,Y_i) - \mu_{\tau}(J) \Big| \geq (1-\alpha) -\mu_{\tau}(J) \Big) =0.$$ Hence we have shown that $\limsup_{n \to \infty} \mathbb{P}(J_{\alpha,\tau}^{(n)} =J)=0$, and this argument holds for any fixed $J \in \ca_{\alpha,\tau}^c$. 

By the union bound and applying this result, $$0 \leq \limsup_{n \to \infty} \mathbb{P}( J_{\alpha,\tau}^{(n)} \in \ca_{\alpha,\tau}^c ) = \limsup_{n \to \infty} \mathbb{P}\Big( \bigcup\limits_{J \in \ca_{\alpha,\tau}^c} \{ J_{\alpha,\tau}^{(n)} =J \} \Big) \leq \limsup_{n \to \infty} \sum_{J \in \ca_{\alpha,\tau}^c} \mathbb{P}(J_{\alpha,\tau}^{(n)} =J) = 0.$$ The above result further implies that $\lim_{n \to \infty} \mathbb{P}(J_{\alpha,\tau}^{(n)} \in \ca_{\alpha,\tau}^c)=0$.

To complete the proof we must show that $\lim_{n \to \infty} \mathbb{P} (J_{\alpha,\tau}^{(n)} =\textnormal{NA})=0$. By Assumption \ref{assump:FeasibleThresh}, $$ \mathbb{E}[g_{[k]}(X,Y)] =\mu_{\tau}([k])  = \mathbb{P} \Big( \max_{j \in [k]} \bigl\{ \iou \big(Y,f(X,t_j) \big) \bigr\} > \tau \Big) > 1-\alpha.$$ Since by Assumption \ref{assump:IID}, $g_{[k]}(X_i,Y_i) \in \{0,1\}$ are IID for $i=1,\dots,n$, by the weak law of large numbers \begin{equation}\label{eq:ConvergenceInProbgAll}
   \frac{1}{n}\sum_{i=1}^n g_{[k]}(X_i,Y_i) \xrightarrow{p} \mu_{\tau}([k]) > 1-\alpha \quad \text{as} \quad n \to \infty. \end{equation} Now for each $n \in \mathbb{Z}_+$, note that $J_{\alpha,\tau}^{(n)} =\textnormal{NA}$ if and only if Algorithm \ref{algorithm-conformal} returns an error in Line \ref{line:find_minimal_set_cover}. Furthermore Algorithm \ref{algorithm-conformal} returns an error in Line \ref{line:find_minimal_set_cover} if and only if $\sum_{i=1}^n g_{[k]}(X_i,Y_i) < (1-\alpha) n$.
   Thus $$\mathbb{P}(J_{\alpha,\tau}^{(n)} =\textnormal{NA}) = \mathbb{P} \Big( \sum_{i=1}^n g_{[k]}(X_i,Y_i) < (1-\alpha)n \Big) \leq \mathbb{P} \Big( \Big| \frac{1}{n}\sum_{i=1}^n g_{[k]}(X_i,Y_i)- \mu_{\tau}([k]) \Big| > \mu_{\tau}([k]) -(1-\alpha) \Big).$$ 
Taking the limit as $n \to \infty$ of each side, by \eqref{eq:ConvergenceInProbgAll} and the definition of convergence in probability, $$\limsup\limits_{n \to \infty} \mathbb{P}(J_{\alpha,\tau}^{(n)} =\text{NA}) \leq  \limsup\limits_{n \to \infty} \mathbb{P} \Big( \Big| \frac{1}{n} \sum_{i=1}^n g_{[k]}(X_i,Y_i) - \mu_{\tau}([k]) \Big| \geq (1-\alpha) -\mu_{\tau}(J) \Big) =0.$$ By nonnegativity of probabilities, $\lim_{n \to \infty} \mathbb{P} (J_{\alpha,\tau}^{(n)} =\textnormal{NA})=0$.
\end{proof}

\subsection{Validity of Algorithm \ref{algorithm-conformal}}

The next proposition formally states the asymptotic validity of Algorithm \ref{algorithm-conformal}.

\begin{proposition}\label{prop:NonadaptiveAlgValidity} Let $C_{\alpha,\tau}^{(n)}(\xt)$ denote the set of segmentation masks returned when running Algorithm \ref{algorithm-conformal} with calibration samples $\big( (X_i,Y_i) \big)_{i=1}^n$ and parameters $\alpha,\tau \in (0,1)$. Under Assumptions \ref{assump:IID} and \ref{assump:FeasibleThresh}, $$\liminf\limits_{n \to \infty} \mathbb{P} \Big(  \max_{\hat{y} \in C_{\alpha,\tau}^{(n)}(\xt) } \bigl\{ \iou( \yt,\hat{y}) \bigr\} > \tau \Big) \geq 1-\alpha.$$
\end{proposition}

\begin{proof}

Fix $\alpha, \tau \in (0,1)$ to be the tolerance and $\iou$ thresholds used in Algorithm \ref{algorithm-conformal}. Let $J_{\alpha,\tau}^{(n)}$ denote the set of indices returned when running Algorithm \ref{algorithm-conformal} with calibration samples $\big( (X_i,Y_i) \big)_{i=1}^n$ and parameters $\alpha,\tau \in (0,1)$, and let $J_{\alpha,\tau}^{(n)}=\textnormal{NA}$ if Algorithm \ref{algorithm-conformal} returns an error. Because either, $J_{\alpha,\tau}^{(n)} \in \ca_{\alpha,\tau}$ and $J_{\alpha,\tau}^{(n)} \in \ca_{\alpha,\tau}^c$ or $J_{\alpha,\tau}^{(n)}=\textnormal{NA}$ with each of these events being mutually exclusive, $$\mathbb{P}(J_{\alpha,\tau}^{(n)} \in \ca_{\alpha,\tau})=1-\mathbb{P}(J_{\alpha,\tau}^{(n)} \in \ca_{\alpha,\tau}^c)- \mathbb{P}(J_{\alpha,\tau}^{(n)} =\text{NA}).$$ By Lemma \ref{lemma:ProbToZero}, $\lim_{n \to \infty} \mathbb{P}(J_{\alpha,\tau}^{(n)}=\textnormal{NA})=0$ and $\lim_{n \to \infty} \mathbb{P}(J_{\alpha,\tau}^{(n)} \in \ca_{\alpha,\tau}^c)=0$, so taking the limit as $n \to \infty$ of each side of the above expression yields that $\lim_{n \to \infty} \mathbb{P}(J_{\alpha,\tau}^{(n)} \in \ca_{\alpha,\tau})=1$.

In Algroithm \ref{algorithm-conformal}, provided that $J_{\alpha,\tau}^{(n)} \neq \textnormal{NA}$, $C_{\alpha, \tau}^{(n)}(\xt)= \{ f(\xt, t_j) \ : \ j \in J_{\alpha,\tau}^{(n)} \}$. Hence, $$\begin{aligned} \mathbb{P} \Big(  \max_{\hat{y} \in C_{\alpha,\tau}^{(n)}(\xt) } \bigl\{ \iou( \yt,\hat{y}) \bigr\} > \tau \Big)
& \geq \mathbb{P} \Big( \Bigl\{ \max_{\hat{y} \in C_{\alpha,\tau}^{(n)}(\xt) } \bigl\{ \iou( \yt,\hat{y}) \bigr\} > \tau \Bigr\} \cap \{ J_{\alpha,\tau}^{(n)} \in \ca_{\alpha,\tau} \} \Big)
\\ & = \mathbb{P} \Big( \Bigl\{ \max_{j \in J_{\alpha,\tau}^{(n)} } \bigl\{ \iou \big( \yt, f(\xt, t_j) \big) \bigr\} > \tau \Bigr\} \cap \{ J_{\alpha,\tau}^{(n)} \in \ca_{\alpha,\tau} \} \Big)
\\ & =  \sum_{J \in \ca_{\alpha,\tau} } \mathbb{P}(J_{\alpha,\tau}^{(n)}=J) \mathbb{P} \Big( \max_{j \in J_{\alpha,\tau}^{(n)} } \bigl\{ \iou \big( \yt, f(\xt, t_j) \big) \bigr\} > \tau \Big|  J_{\alpha,\tau}^{(n)}=J \Big)
\\ & =  \sum_{J \in \ca_{\alpha,\tau} } \mathbb{P}(J_{\alpha,\tau}^{(n)}=J) \mathbb{P} \Big( \max_{j \in J } \bigl\{ \iou \big( \yt, f(\xt, t_j) \big) \bigr\} > \tau \Big|  J_{\alpha,\tau}^{(n)}=J \Big)
\\ & =  \sum_{J \in \ca_{\alpha,\tau} } \mathbb{P}(J_{\alpha,\tau}^{(n)}=J) \mathbb{P} \Big( \max_{j \in J } \bigl\{ \iou \big( \yt, f(\xt, t_j) \big) \bigr\} > \tau \Big)
\\ & = \sum_{J \in \ca_{\alpha,\tau} } \mathbb{P}(J_{\alpha,\tau}^{(n)}=J) \mu_{\tau}(J)
\\ & \geq (1-\alpha) \sum_{J \in \ca_{\alpha,\tau} } \mathbb{P}(J_{\alpha,\tau}^{(n)}=J)
\\ & = (1-\alpha) \mathbb{P}( J_{\alpha,\tau}^{(n)} \in \ca_{\alpha,\tau}).
\end{aligned}$$ Above the third step follows from the law of total probability, the fifth step holds because $J_{\alpha,\tau}^{(n)}$ is independent of $(\xt,\yt)$ by Assumption \ref{assump:IID} (with the former being a function of the calibration sample $\big((X_i,Y_i)\big)_{i=1}^n$), and the sixth and seventh steps following from the definition of $\mu_{\tau}(J)$ and $\ca_{\alpha,\tau}$ respectively. Recalling that $\lim_{n \to \infty} \mathbb{P}(J_{\alpha,\tau}^{(n)} \in \ca_{\alpha,\tau})=1$, we can take the $\liminf$ as $n \to \infty$ of each side of the above expression to get that $$\liminf_{n \to \infty} \mathbb{P} \Big(  \max_{\hat{y} \in C_{\alpha,\tau}^{(n)}(\xt) } \bigl\{ \iou( \yt,\hat{y}) \bigr\} > \tau \Big) \geq (1-\alpha) \cdot \liminf_{n \to \infty} \mathbb{P}( J_{\alpha,\tau}^{(n)} \in \ca_{\alpha,\tau}) =(1-\alpha).$$

\end{proof}

\subsection{Proof of Theorem \ref{theorem:ValidityMainTheorem}}\label{sec:ProofOfMainAlgortihmValidity}

For convenience we restate Theorem \ref{theorem:ValidityMainTheorem} from the main text below. Subsequently, a proof is provided.

 \textit{Let $C_{\alpha,\tau,\eta}^{(n)}(\xt)$ and $\tilde{\theta}^{(n)}$ denote the set of segmentation masks and $\iou$ threshold returned when running Algorithm \ref{algorithm-remove-duplicates} with calibration samples $\big( (X_i,Y_i) \big)_{i=1}^n$ and parameters $\alpha,\tau,\eta \in (0,1)$. Under Assumptions \ref{assump:IID}, \ref{assump:TieBreakingDeterministic}, and \ref{assump:StrictFeasibilitySmallestSet}, $$\liminf\limits_{n \to \infty} \mathbb{P} \Big(  \max_{\hat{y} \in C_{\alpha,\tau,\eta}^{(n)}(\xt) } \bigl\{ \iou( \yt,\hat{y}) \bigr\} \geq \tilde{\theta}^{(n)} \Big) \geq 1-\alpha.$$}

\begin{proof}
Fix $\alpha, \tau,\eta \in (0,1)$ to be the tolerance and $\iou$ thresholds used in Algorithm \ref{algorithm-remove-duplicates}. Let $J_{\alpha,\tau}^{(n)}$ denote the set of indices returned when running Algorithm \ref{algorithm-conformal} (in Line \ref{line:CallFirstAlg} of Algorithm \ref{algorithm-remove-duplicates}) with calibration samples $\big( (X_i,Y_i) \big)_{i=1}^n$ and parameters $\alpha,\tau \in (0,1)$, and let $J_{\alpha,\tau}^{(n)}=\textnormal{NA}$ if Algorithm \ref{algorithm-conformal} returns an error.

Recall that $J_{\alpha,\tau}^*$ is a fixed element of $\argmin_{J \subseteq \ca_{\alpha,\tau}} \{ \vert J \vert \}$, which under Assumption \ref{assump:TieBreakingDeterministic} was set to be the element of $\ca_{\alpha,\tau}$ that is first among all elements in $\ca_{\alpha,\tau}$ with the minimum cardinality, according to the ordering used in Line \ref{line:find_minimal_set_cover} of Algorithm \ref{algorithm-conformal}. We will first show that $\lim_{n \to \infty} \mathbb{P}(J_{\alpha,\tau}^{(n)}=J_{\alpha,\tau}^*)=1$. To do this, let $$\ca_{\alpha,\tau}^{(n)} = \Bigl\{ J \subseteq [k] \ : \ \sum_{i=1}^n \mathbbm{1} \bigl\{ \max_{j \in J} \iou \big(Y_i, f(X_i,t_j) \big) > \tau\bigr\} \geq (1-\alpha)n  \Bigr\}$$ and note that in Line \ref{line:find_minimal_set_cover} of Algorithm \ref{algorithm-conformal}, $J_{\alpha,\tau}^{(n)}$ is set to the smallest element of $\ca_{\alpha,\tau}^{(n)}$ (where the preset ordering is used if $\ca_{\alpha,\tau}^{(n)}$ has multiple sets of the same cardinality). Because $J_{\alpha,\tau}^{(n)}$ and $J_{\alpha,\tau}^*$ are the smallest sets in $\ca_{\alpha,\tau}^{(n)}$ and $\ca_{\alpha,\tau}$ (with ties between sets of the same size broken according to the same ordering), $J_{\alpha,\tau}^* \in \ca_{\alpha,\tau}^{(n)}$ implies that either $J_{\alpha,\tau}^{(n)} \notin \ca_{\alpha,\tau}$ or $J_{\alpha,\tau}^{(n)}=J_{\alpha,\tau}^*$. Hence if both
$J_{\alpha,\tau}^{(n)} \in \ca_{\alpha,\tau}$ and $ J_{\alpha,\tau}^* \in \ca_{\alpha,\tau}^{(n)}$ then $J_{\alpha,\tau}^{(n)}=J_{\alpha,\tau}^*$. Thus by monotonicity of probability measure and the union bound, 

$$\begin{aligned}
    \mathbb{P}(J_{\alpha,\tau}^{(n)}=J_{\alpha,\tau}^*) & \geq \mathbb{P} \big( \{ J_{\alpha,\tau}^{(n)} \in \ca_{\alpha,\tau} \} \cap \{ J_{\alpha,\tau}^* \in \ca_{\alpha,\tau}^{(n)} \}\big) 
    \\ & = 1 - \mathbb{P}\big( \{ J_{\alpha,\tau}^{(n)} \in \ca_{\alpha,\tau} \}^c \cup \{ J_{\alpha,\tau}^* \in \ca_{\alpha,\tau}^{(n)} \}^c \big)
    \\ & \geq 1 - \mathbb{P}\big( \{ J_{\alpha,\tau}^{(n)} \in \ca_{\alpha,\tau} \}^c \big) -\mathbb{P} \big( \{ J_{\alpha,\tau}^* \in \ca_{\alpha,\tau}^{(n)} \}^c \big)
    \\ & = 1  - \mathbb{P}\big(  J_{\alpha,\tau}^{(n)} \in \ca_{\alpha,\tau}^c  \big)-\mathbb{P}\big(  J_{\alpha,\tau}^{(n)} = \text{NA}  \big)-\mathbb{P} \big( \{ J_{\alpha,\tau}^* \in \ca_{\alpha,\tau}^{(n)} \}^c \big).
\end{aligned}.$$

Let $\cx =\mathbb{R}^{W \times H \times C} \times [W] \times [H]$ and $\cy =\{0,1\}^{W \times H}$ denote the input and output space for $f(\cdot,t)$, respectively, and define $h : \cx \times \cy \to \{0,1\}$ by $$h(x,y) \equiv \begin{cases}
    1 & \text{if } \iou \big(y , f(x,t_j) \big) > \tau \text{ for some } j \in J_{\alpha,\tau}^*, 
    \\ 0 & \text{otherwise.} \end{cases}$$ Observe that by Assumption \ref{assump:StrictFeasibilitySmallestSet}, $$\e[h(X,Y)]= \mathbb{P} \Big( \max_{j \in J_{\alpha,\tau}^*} \bigl\{ \iou \big(Y,f(X,t_j) \big) \bigr\} > \tau \Big) = \mu_{\tau}(J_{\alpha,\tau}^*) > 1-\alpha,$$ where the penultimate step follows by definition of $\mu_{\tau}(J)$ in Equation \eqref{eq:MuJDef} and the last step follows from Assumption \ref{assump:StrictFeasibilitySmallestSet}. Since by Assumption \ref{assump:IID}, $h(X_i,Y_i) \in \{0,1\}$ are IID for $i=1,\dots,n$, by the weak law of large numbers \begin{equation}\label{eq:ConvergenceInProbH}
   \frac{1}{n}\sum_{i=1}^n h(X_i,Y_i) \xrightarrow{p} \mu_{\tau}(J_{\alpha,\tau}^*) > 1-\alpha \quad \text{as} \quad n \to \infty. \end{equation} Since, by definition of $h(\cdot,\cdot)$ and $\ca_{\alpha,\tau}^{(n)}$,  $\sum_{i=1}^n h(X_i,Y_i) \geq (1-\alpha) n$ if and only if $J_{\alpha,\tau}^* \in \ca_{\alpha,\tau}^{(n)}$, $$\mathbb{P} \big( \{ J_{\alpha,\tau}^* \in \ca_{\alpha,\tau}^{(n)} \}^c \big) = \mathbb{P} \Big( \frac{1}{n} \sum_{i=1}^n h(X_i,Y_i) < (1-\alpha) \Big) \leq \mathbb{P} \Big(  \Big| \frac{1}{n} \sum_{i=1}^n h(X_i,Y_i) - \mu_{\tau}(J_{\alpha,\tau}^*) \Big| \geq \mu_{\tau}(J_{\alpha,\tau}^*)-(1-\alpha) \Big).$$ Taking the limit as $n \to \infty$ of each side, by \eqref{eq:ConvergenceInProbH} and the definition of convergence in probability, $$\limsup\limits_{n \to \infty} \mathbb{P} \big( \{ J_{\alpha,\tau}^* \in \ca_{\alpha,\tau}^{(n)} \}^c \big)  \leq  \limsup\limits_{n \to \infty} \mathbb{P} \Big(  \Big| \frac{1}{n} \sum_{i=1}^n h(X_i,Y_i) - \mu_{\tau}(J_{\alpha,\tau}^*) \Big| \geq \mu_{\tau}(J_{\alpha,\tau}^*)-(1-\alpha) \Big)=0.$$ By nonnegativity of probabilities $\lim_{n \to \infty} \mathbb{P} \big( \{ J_{\alpha,\tau}^* \in \ca_{\alpha,\tau}^{(n)} \}^c \big)=0$. Also note that since Assumption \ref{assump:StrictFeasibilitySmallestSet} implies Assumption \ref{assump:FeasibleThresh}, by Lemma \ref{lemma:ProbToZero}, $\lim_{n \to \infty} \mathbb{P}(J_{\alpha,\tau}^{(n)}=\textnormal{NA})=0$ and $\lim_{n \to \infty} \mathbb{P}(J_{\alpha,\tau}^{(n)} \in \ca_{\alpha,\tau}^c)=0$. Combining these results with an earlier inequality, $$\begin{aligned}
       \liminf_{n \to \infty} \mathbb{P}(J_{\alpha,\tau}^{(n)}=J_{\alpha,\tau}^*) \geq 1   - \limsup_{n \to \infty} \mathbb{P}\big(  J_{\alpha,\tau}^{(n)} \in \ca_{\alpha,\tau}^c  \big)- \limsup_{n \to \infty}\mathbb{P}\big(  J_{\alpha,\tau}^{(n)} = \text{NA}  \big)- \limsup_{n \to \infty}\mathbb{P} \big( \{ J_{\alpha,\tau}^* \in \ca_{\alpha,\tau}^{(n)} \}^c \big)=1.
   \end{aligned}$$
Since probabilities are upper bounded by $1$, $\lim_{n \to \infty} \mathbb{P}(J_{\alpha,\tau}^{(n)}=J_{\alpha,\tau}^*)=1$.

We are now ready to investigate the limiting behavior of Algorithm \ref{algorithm-remove-duplicates}. To do this, let $\mathcal{P}_k$ denote the collection of subsets of $[k]$. Observe that, $$\mathbb{P} \Big(  \max_{\hat{y} \in C_{\alpha,\tau,\eta}^{(n)}(\xt) } \bigl\{ \iou( \yt,\hat{y}) \bigr\} \geq \tilde{\theta}^{(n)} \Big)  = \sum_{J \in \mathcal{P}_k \cup \{\text{NA}\}} \mathbb{P} \Big( \Bigl\{  \max_{\hat{y} \in C_{\alpha,\tau,\eta}^{(n)}(\xt) } \bigl\{ \iou( \yt,\hat{y}) \bigr\} \geq \tilde{\theta}^{(n)} \Bigr\} \cap \{J_{\alpha,\tau}^{(n)}=J \} \Big).$$  Taking the limit as $n \to \infty$ of each side and noting that for $J \neq J_{\alpha,\tau}^*$, $\lim_{n \to \infty} \mathbb{P}(J_{\alpha,\tau}^{(n)}=J)=0$ (because $\lim_{n \to \infty} \mathbb{P}(J_{\alpha,\tau}^{(n)}=J_{\alpha,\tau}^*)=1$), and hence $\lim_{n \to \infty} \mathbb{P} \big( E_n \cap \{J_{\alpha,\tau}^{(n)}=J\}  \big)=0$ for any sequence of events $E_n$ and $J \neq J_{\alpha,\tau}^*$, it follows that $$\begin{aligned} \lim_{n \to \infty} \mathbb{P} \Big(  \max_{\hat{y} \in C_{\alpha,\tau,\eta}^{(n)}(\xt) } \bigl\{ \iou( \yt,\hat{y}) \bigr\} \geq \tilde{\theta}^{(n)} \Big) &  = \lim_{n \to \infty}  \mathbb{P} \Big( \Bigl\{  \max_{\hat{y} \in C_{\alpha,\tau,\eta}^{(n)}(\xt) } \bigl\{ \iou( \yt,\hat{y}) \bigr\} \geq \tilde{\theta}^{(n)} \Bigr\} \cap \{J_{\alpha,\tau}^{(n)}=J_{\alpha,\tau}^* \} \Big).\end{aligned}$$

To simplify the above expression let $\mathcal{B}$ denote the set of all possible collections of binary masks in $\cy$. Define $u : \cx \times \mathcal{P}_k \to \mathcal{B}$ be the function such that $$u(x,J) \equiv \texttt{UNIQUE} \big( \{ f(x,t_j) \ : \ j \in J \}, \eta \big) \quad \text{for all } x \in \cx \quad \text{and} \quad J \subseteq [k],$$ where \texttt{UNIQUE} is the procedure defined in Algorithm \ref{algorithm-remove-duplicates}. Note that by Assumption \ref{assump:TieBreakingDeterministic}, the procedure  \texttt{UNIQUE} is deterministic and hence $u(\cdot,\cdot)$ is a deterministic function. Next, define $M: \cx \times \cy \times \mathcal{P}_k \to [0,1]$ to be the function given by $$M(x,y,J) \equiv \max_{\hat{y} \in u(x,J)} \bigl\{ \iou(y,\hat{y}) \bigr\} \quad \text{for all } x \in \cx, \quad y \in \cy, \quad \text{and} \quad J \subseteq [k].$$ Note that $M: \cx \times \cy \times \mathcal{P}_k \to [0,1]$ is also a deterministic function. Finally let $\cq_{\alpha}(\cdot)$ be a deterministic operator that takes in a sequence of real numbers as input and returns the $\alpha$ empirical quantile (as defined in Definition 2.9 of \citet{angelopoulos2025theoreticalfoundationsconformalprediction}). Now observe that by Algorithm \ref{algorithm-remove-duplicates}$$\max_{\hat{y} \in C_{\alpha,\tau,\eta}^{(n)}(\xt) } \bigl\{ \iou( \yt,\hat{y}) \bigr\} = M(\xt,\yt,J_{\alpha,\tau}^{(n)}) \quad \text{and} \quad \tilde{\theta}^{(n)} = \cq_{\alpha} \Big( \bigl\{ M(X_i,Y_i, J_{\alpha,\tau}^{(n)}) \bigr\}_{i=1}^n \Big),$$ provided that $J_{\alpha,\tau}^{(n)} \neq \text{NA}.$ Hence letting $o(1)$ denote any term that converges to $0$ as $n \to \infty$ and letting $B_n$ and $E_n$ denote the following events $$B_n \equiv  \Bigl\{  M(\xt,\yt, J_{\alpha,\tau}^*) \geq \cq_{\alpha} \Big( \bigl\{ M(X_i,Y_i, J_{\alpha,\tau}^*) \bigr\}_{i=1}^n \Big) \Bigr\} \quad \text{and} \quad E_n \equiv \{ J_{\alpha,\tau}^{(n)}=J_{\alpha,\tau}^*\},$$ an earlier result simplifies to $$\begin{aligned}  \mathbb{P} \Big(  \max_{\hat{y} \in C_{\alpha,\tau,\eta}^{(n)}(\xt) } \bigl\{ \iou( \yt,\hat{y}) \bigr\} \geq \tilde{\theta}^{(n)} \Big) &  =  \mathbb{P} \Big( \Bigl\{  \max_{\hat{y} \in C_{\alpha,\tau,\eta}^{(n)}(\xt) } \bigl\{ \iou( \yt,\hat{y}) \bigr\} \geq \tilde{\theta}^{(n)} \Bigr\} \cap \{J_{\alpha,\tau}^{(n)}=J_{\alpha,\tau}^* \} \Big) +o(1)
\\ & =  \mathbb{P} \Big( \Bigl\{  M(\xt,\yt, J_{\alpha,\tau}^{(n)}) \geq \cq_{\alpha} \Big( \bigl\{ M(X_i,Y_i, J_{\alpha,\tau}^{(n)}) \bigr\}_{i=1}^n \Big) \Bigr\} \cap E_n \Big) +o(1) 
\\ & =  \mathbb{P} \Big( \Bigl\{  M(\xt,\yt, J_{\alpha,\tau}^*) \geq \cq_{\alpha} \Big( \bigl\{ M(X_i,Y_i, J_{\alpha,\tau}^*) \bigr\}_{i=1}^n \Big) \Bigr\} \cap E_n \Big) +o(1)
\\ & = \mathbb{P}(B_n \cap E_n) +o(1). 
\end{aligned}$$ Above the penultimate step holds because $J_{\alpha,\tau}^{(n)}=J_{\alpha,\tau}^*$ under the event $E_n$. Now we will show that for each $n \in \mathbb{Z}_+$, $\mathbb{P}(B_n^c) < \alpha$. To do this fix $n \in \mathbb{Z}_+$, and define $\xi_{n+1} \equiv M(\xt,\yt, J_{\alpha,\tau}^*)$, and $\xi_i \equiv M(X_i,Y_i, J_{\alpha,\tau}^*)$ for each $i=1,\dots,n$. Since $J_{\alpha,\tau}^*$ is a fixed set of indices and by Assumption \ref{assump:IID}, $\big( (X_i,Y_i) \big)_{i=1}^n \stackrel{\text{iid}}{\sim} \mathbb{P}$ and independently $(\xt,\yt) \sim \mathbb{P}$, it follows that $(\xi_i)_{i=1}^{n+1}$ is an exchangeable sequence of random variables. Note that $$\mathbb{P}(B_n^c) = \mathbb{P} \Big( M(\xt,\yt, J_{\alpha,\tau}^*) < \cq_{\alpha} \Big( \bigl\{ M(X_i,Y_i, J_{\alpha,\tau}^*) \bigr\}_{i=1}^n \Big)  \Big) =\mathbb{P} \Big( \xi_{n+1} < \cq_{\alpha} \big( \bigl\{ \xi_i \bigr\}_{i=1}^n \big)  \Big) < \alpha.$$ Above the last step follows from Fact 2.15(ii) in \cite{angelopoulos2025theoreticalfoundationsconformalprediction} because $(\xi_i)_{i=1}^{n+1}$ is an exchangeable sequence of random variables. Hence we have shown that $\mathbb{P}(B_n^c)< \alpha$ and this argument holds for any $n \in \mathbb{Z}_+$.  Since probability measures are upper bounded by $1$ we can apply this result to get that that for all $n \in \mathbb{Z}_+$, $$\mathbb{P}(B_n \cap E_n) = \mathbb{P}(B_n)+\mathbb{P}(E_n)-\mathbb{P}(B_n \cup E_n) \geq \mathbb{P}(B_n)+\mathbb{P}(E_n)-1=\mathbb{P}(E_n)-\mathbb{P}(B_n^c)>\mathbb{P}(E_n)-\alpha.$$ Since we have shown that  $\lim\limits_{n \to \infty} \mathbb{P}(E_n)=\lim\limits_{n \to \infty} \mathbb{P}(J_{\alpha,\tau}^{(n)}=J_{\alpha,\tau}^*)=1$, we can take the limit as $n \to \infty$ of each side of the above inequality to get that $\liminf_{n \to \infty} \mathbb{P}(B_n \cap E_n) \geq 1-\alpha$. Combining this with an earlier result we get that $$\liminf_{n \to \infty} \mathbb{P} \Big(  \max_{\hat{y} \in C_{\alpha,\tau,\eta}^{(n)}(\xt) } \bigl\{ \iou( \yt,\hat{y}) \bigr\} \geq \tilde{\theta}^{(n)} \Big) = \liminf_{n \to \infty} \Big( \mathbb{P}(B_n \cap E_n) +o(1) \Big) = \liminf_{n \to \infty}  \mathbb{P}(B_n \cap E_n) \geq 1-\alpha.$$


\end{proof}

\newpage
\section{IoU histograms for calibration data}
Our conformal algorithm requires the selection of a target error rate $\alpha$ and IoU threshold $\tau$. To guide our choices, we plotted histograms of the maximum IoU achievable for each calibration data point under any value of the tunable parameter $T$ (Figure \ref{fig:calibration-iou-histograms}). For each experiment, we minimized the target error rate $\alpha$ while maximizing IoU threshold $\tau$, subject to feasibility on the calibration dataset. 

\begin{figure*}[h!]
  \centering
  \includegraphics[width=0.9\textwidth]{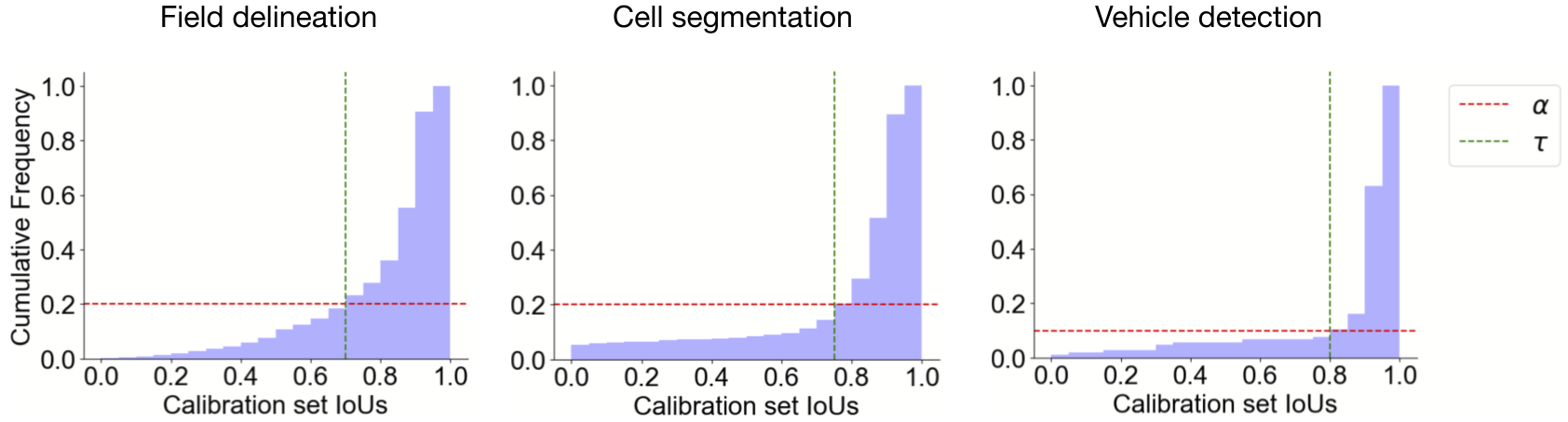}
  \caption{\textbf{Cumulative distributions of maximum IoUs over all $T$ values, for calibration data points.} For each experiment, we minimized the target error rate $\alpha$ while maximizing IoU threshold $\tau$, subject to feasibility (i.e., there must be less than $\alpha$ fraction of calibration points with $\text{IoU} < \tau$).}\label{fig:calibration-iou-histograms}
\end{figure*}

\newpage
\section{Why the naive best parameter baseline has low coverage}\label{appendix:naive-best-parameter-baseline}
While our conformal method outputs a confidence set of multiple parameter values, the naive best parameter baseline outputs the single parameter value that results in the highest coverage on the calibration dataset. The disadvantage of the naive best parameter method is that there does not always exist a single parameter value that results in accurate predictions for most calibration points. In particular, in our instance segmentation setting, there does not always exist a single parameter value that results in mask predictions with high IoU for $(1-\alpha)$ fraction of the calibration points. Thus, the naive best parameter baseline often cannot guarantee high coverage on the test set. 

As an illustrative example (Figure \ref{fig:single-parameter-baseline}), we take the field delineation dataset and plot the fraction of calibration predictions with IoU greater than $\tau = 0.7$ at each single parameter value $T$ (which ranges from 0 to 1). The target coverage for the field delineation example in our paper was $(1-\alpha) = 0.8$, shown as a red line in the figure. We see that there is no single parameter value that attains the target coverage of 80\% on the calibration set; the best parameter value ($T=0.24$) results in about 67\% coverage. Since the naive best parameter baseline only outputs a single parameter value, it will fail to reach the target coverage on the calibration set and is thus unlikely to reach the target coverage on the test set. In contrast, our conformal method constructs a set of multiple parameter values such that for about 80\% of the test points, at least one value will result in high IoU. 

\begin{figure*}[h!]
  \centering
  \includegraphics[width=0.35\textwidth]{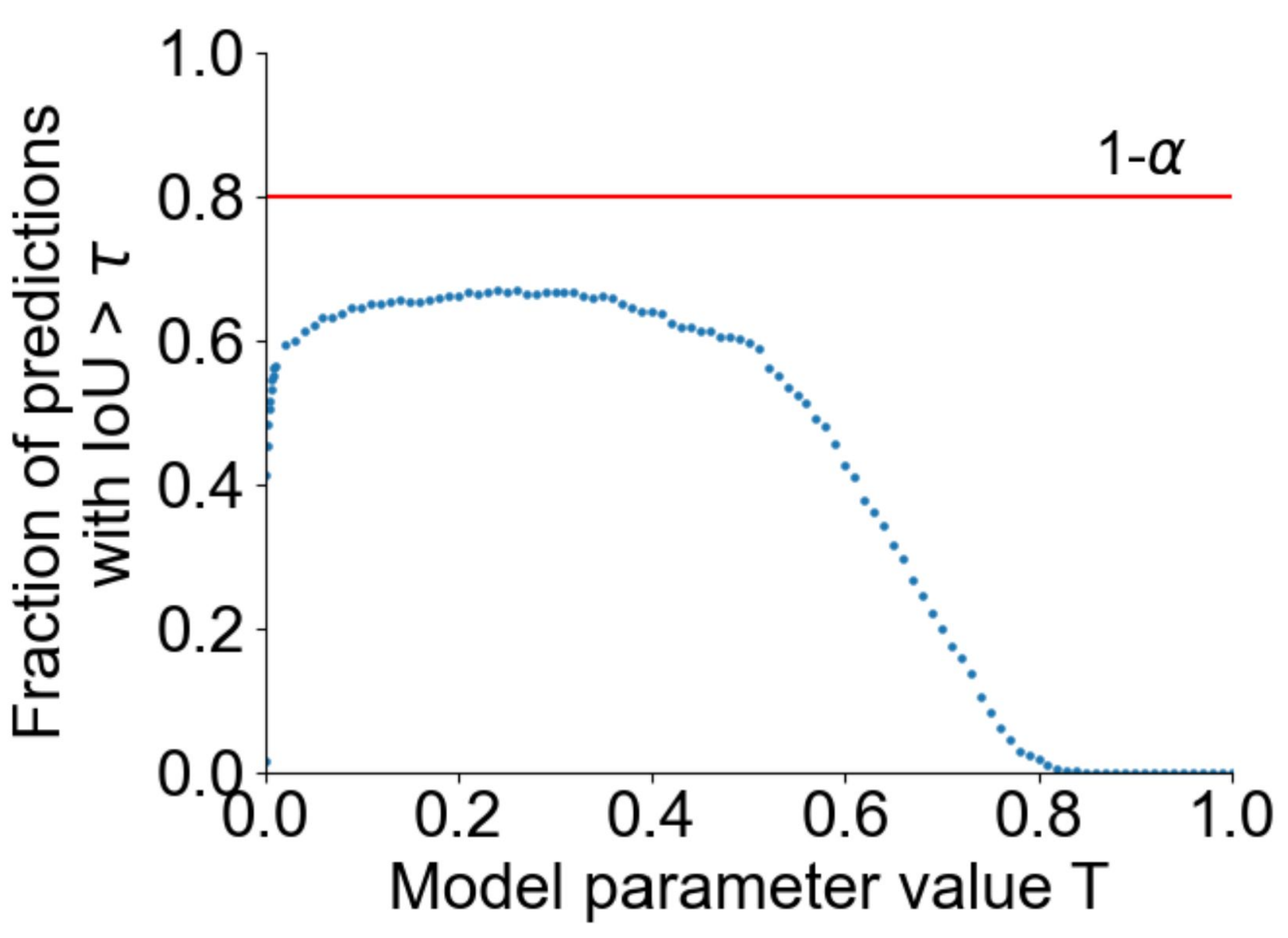}
  \caption{\textbf{Fraction of calibration predictions with IoU greater than $\tau = 0.7$ at each model parameter value $T$ for the field delineation dataset.} No single parameter value attains the target coverage of 80\% on the calibration set.}\label{fig:single-parameter-baseline}
\end{figure*}

\newpage
\section{Comparison to morphological dilation baseline}\label{appendix:dilation-baseline}
We compare our method to the morphological dilation-based method from \citet{mossina2025conformal}. This method constructs conformal confidence sets for binary segmentation using dilation, i.e., adding a margin to a predicted mask, where the margin size is a fixed number determined using calibration data. This algorithm only guarantees that the true mask is contained within the dilated mask with high probability, and does not guarantee high IoU.

\paragraph{Experimental setup} For the comparison, we use our field delineation dataset. For each sampled pixel, we first generate a single mask prediction at the best parameter value as defined in the previous paragraph, which is $T=0.24$. For each calibration pixel, we dilate the predicted mask by $D$ pixels, for all values of $D$ from 0 to 30, and record the smallest $D$ value where the dilated mask contains the true field. Finally we compute the 80th percentile of these minimum $D$ values, which is $D=2$. The conformal guarantee is that for a test pixel, with 80\% probability, dilating the predicted mask by $D=2$ pixels will result in a mask that contains the true field.

\paragraph{Results} The dilation-based method results in low coverage for IoU, undersegmentation, and lack of flexibility in predictions. We visualize results for five example test points (Figure \ref{fig:dilation-images}). Three of the five dilated masks have IoU less than 0.7. Overall, only 54.4\% of the 917 test points have dilated masks with IoU greater than 0.7. This low coverage is expected as the dilation-based method does not guarantee high IoUs, as discussed above. Our paper’s method achieves 79.7\% coverage. The dilated masks tend to undersegment the fields because the dilation-based method is designed to output a single mask that contains the true field with high probability (see Appendix \ref{appendix:structural-uncertainty-metrics} for a quantitative comparison of the undersegmentation against our method). Finally, each dilated mask is constrained to be a modification of a single prediction, while our method’s conformal sets can contain a diverse range of possible fields from multiple predictions, allowing users to select from multiple plausible options. 

\begin{figure*}[h!]
  \centering
  \includegraphics[width=0.5\textwidth]{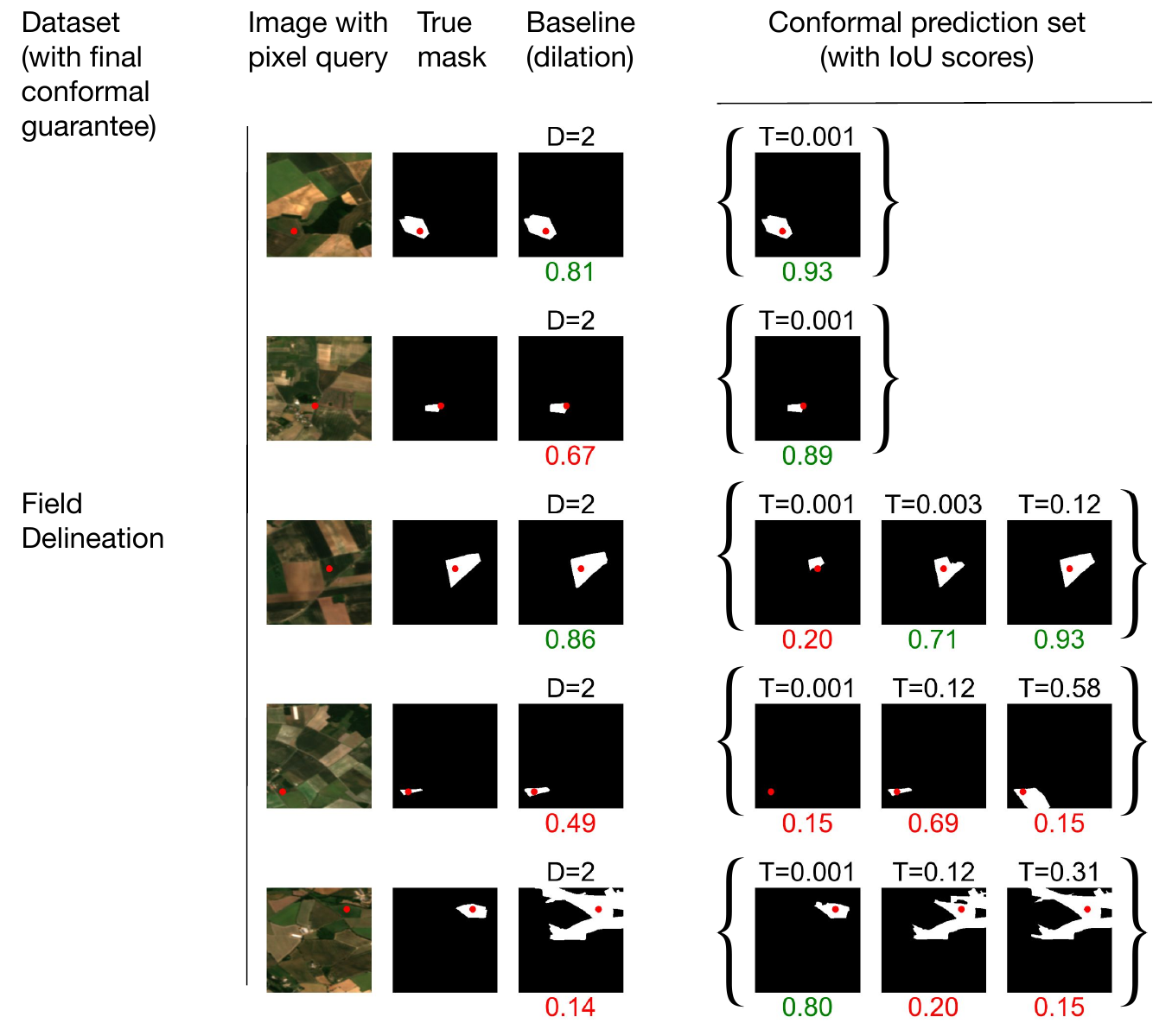}
  \caption{\textbf{Conformal prediction sets for example test queries, compared to dilation baseline predictions.} The third column shows the mask prediction for each test point dilated by $D=2$ pixels, while the rightmost column shows the conformal prediction sets generated by our method. While most of our conformal sets contain at least one mask with high IoU (green) with the ground truth, the dilation-based method does not provide guarantees on IoU.}\label{fig:dilation-images}
\end{figure*}

\newpage
\section{Quantifying structural uncertainty}\label{appendix:structural-uncertainty-metrics} 

We are interested in capturing structural uncertainty in the size and shape of objects, including over- and under-segmentation: in applications such as field delineation, there is often uncertainty about whether adjacent regions should be treated as a single object or split into multiple objects. Qualitatively we find that our prediction sets contain masks of diverse sizes and shapes, while previous works only predict a single mask or modifications of a single baseline mask (e.g., dilation or boundary expansion). 

To quantify our method’s performance, we will use the oversegmentation score $S_{\text{over}}$ and undersegmentation score $S_{\text{under}}$ from \citet{persello2009novel}. For a true mask $Y$ and predicted mask $\hat{Y}$, they are defined as follows:
\begin{align*}
    S_{\text{over}} = 1 - |Y \cap \hat{Y}| / |Y| \\
    S_{\text{under}} = 1 - |Y \cap \hat{Y}| / |\hat{Y}|
\end{align*}
where the absolute value sign denotes the area of a mask. The scores range from 0 to 1, where 0 is a perfect match and higher scores correspond to more over- and under-segmentation respectively.

For each test point in the field delineation dataset, we compute these scores between the true mask and each of the masks in our conformal prediction set. We record the minimum oversegmentation and undersegmentation score over the masks in each prediction set. We also compute these scores for the dilation-based conformal baseline (Appendix \ref{appendix:dilation-baseline}). We plot histograms of the scores, comparing our method with the baseline (Figure \ref{fig:segmentation-scores}). We find that both methods have low oversegmentation, but the baseline method suffers from undersegmentation compared to our method. This is because the dilation-based method is designed to output a conformal mask that contains the true field, resulting in field predictions that are too large. In contrast, our method is designed to guarantee high IoU by outputting a set of multiple predictions, so that there is likely at least one prediction that avoids undersegmentation (as well as at least one prediction that avoids oversegmentation).

\begin{figure*}[h!]
  \centering
  \includegraphics[width=0.6\textwidth]{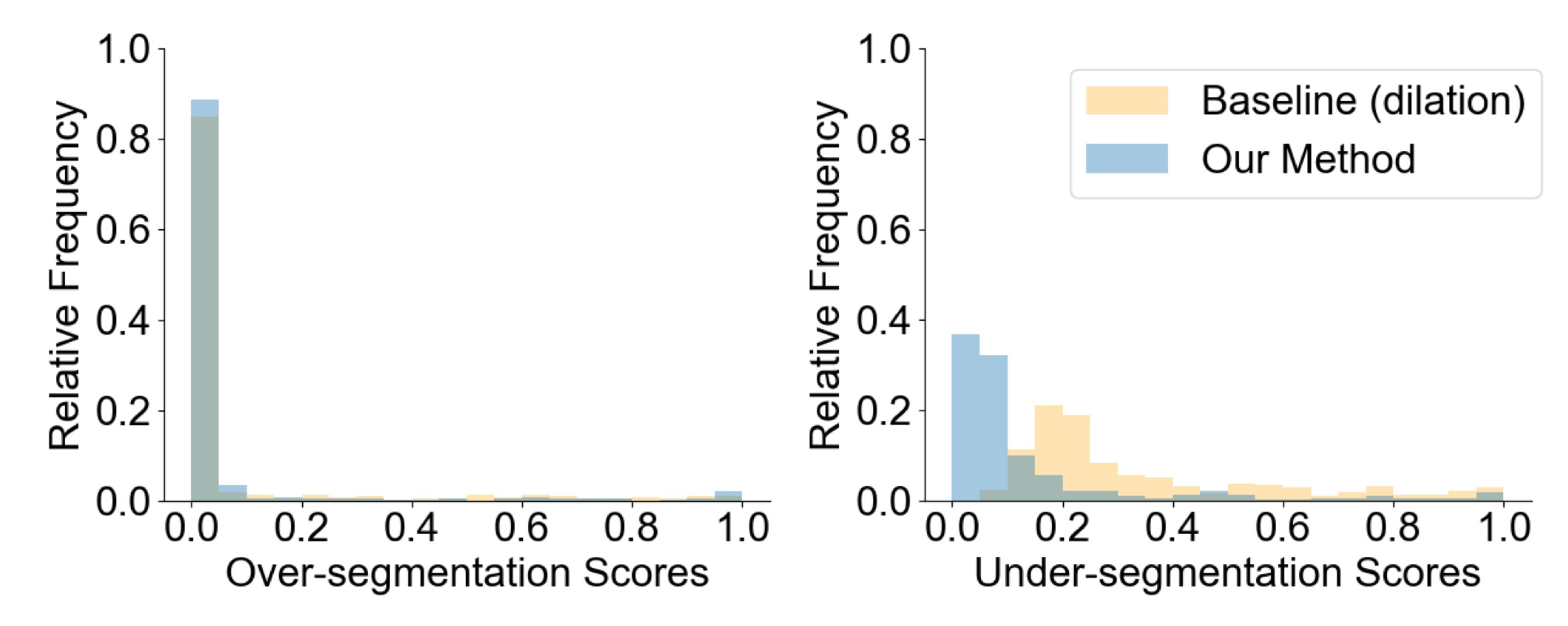}
  \caption{\textbf{Over- and under-segmentation scores for our method and the dilation-based conformal baseline, over the field delineation test set.} Both methods have low oversegmentation, but the baseline method suffers from undersegmentation.}\label{fig:segmentation-scores}
\end{figure*}

\newpage
\section{Modified conformal algorithm with finite sample guarantees}\label{appendix:finite-sample-guarantees}

\begin{algorithm}[!hbt]
\small
\caption{Compute conformal prediction set for instance segmentation, with finite sample guarantees}\label{algorithm-finite-sample-guarantees}
\begin{multicols}{2}
\begin{algorithmic}
\State \textbf{Input:} 
\begin{itemize}
    \item calibration data, randomly split into two sets of size $n_1$ and $n_2$: $\{(X_1,Y_1),\dots,(X_{n_1}, Y_{n_1})\}$ and $\{(X_{n_1+1},Y_{n_1+1}),\dots,(X_{n_1+n_2}, Y_{n_1+n_2})\}$ where each image has width $W$ and height $H$
    \item test input $X^\text{test}$
    \item ML model $f(X, T)$
    \item parameter values $t_1,\dots,t_k$
    \item error rate $0 < \alpha < 1$
    \item IoU threshold $0 < \tau < 1$
    \item duplicate prediction IoU threshold $0 < \eta < 1$
\end{itemize}
\State \textbf{Output:} 
\begin{itemize}
    \item conformal prediction set with duplicates removed $C_{\hat{\lambda}}(X^\text{test})$
\end{itemize}
\State
\end{algorithmic}
\begin{algorithmic}[1]
\State \textit{// duplicate removal function}
\Procedure{Unique*}{$C$, $\eta$} 
    \State $m = |C|$
    \If{$m = 1$}
        \Return $C$
    \Else
        \State $C' \gets \Call{Unique*}{C[1:m-1], \eta}$
        \If{there exists $\hat{Y}' \in  C'$ such that $\text{IoU}(C[m], \hat{Y}') > \eta$}
            \Return $C'$
        \Else{}
            \Return $C' \cup \{C[m]\}$
        \EndIf
    \EndIf
\EndProcedure
\State
\State \textit{// using $n_1$ datapoints: compute set $S_j$ of data indices ``covered" by each parameter index $j$ (i.e., datapoints where predicted mask $f(X_i, t_j)$ has IoU $> \tau$)}
\For{$j=1,2,\ldots,k$}
    \For{$i=1,2,\ldots,n_1$}
        \State $\rho_{ij} \gets \text{IoU}(Y_i, f(X_i, t_j))$
    \EndFor
    \State $S_j \gets \{i : \rho_{ij} > \tau \}$
\EndFor
\columnbreak
\State \textit{// create ordered list $L$ of parameter indices from highest to lowest priority via greedy algorithm}
\State $L \gets []$
\For{$j'=1,2,\ldots,k$}
    \State \textit{// find parameter index that covers greatest number of previously un-covered datapoints}
    \State $L[j'] \gets \argmax_{j\in[k]} |S_{j}|$ 
    \For{$j''=1,2,\ldots,k$}
        \State $S_{j''} \gets S_{j''} \setminus S_{L[j']}$
    \EndFor
\EndFor
\State 
\State \textit{// using remaining $n_2$ datapoints: use Conformal Risk Control (CRC) to compute smallest $\hat{\lambda} \in [k]$ such that the first $\hat{\lambda}$ parameter indices in $L$ result in a valid conformal prediction set}
\State $\hat{\lambda} \gets k+1$
\For{$\lambda=1,2,\ldots,k$}
    \For{$i=n_1+1,\ldots,n_1+n_2$}
        \State \textit{// compute ordered list of predictions using first $\lambda$ parameter indices in $L$ and remove duplicates}
        \State $C_{\lambda}(X_i) \gets \{ f(X_i, t_j) : j \in L[1:\lambda]\}$
        \State $C_{\lambda}(X_i)  \gets \Call{Unique*}{C_{\lambda}(X_i), \eta}$ 
        \State \textit{// compute loss function for whether $C_{\lambda}(X_i)$ contains a high IoU prediction}
        \State $l_i(\lambda) \gets \mathbbm{1}(\max_{\hat{Y}_i \in C_{\lambda}(X_i) } \text{IoU}(Y_i, \hat{Y}_i) \leq \tau)$
    \EndFor
    \State $\hat{r}(\lambda) \gets \frac{1}{n_2}\sum_{i=n_1+1}^{n_1+n_2}l_i(\lambda)$ \textit{// empirical risk}
    \If{$\hat{r}(\lambda) \leq \alpha - \frac{1-\alpha}{n_2}$} 
        \State $\hat{\lambda} \gets \lambda$
        \State \textbf{break}
    \EndIf
\EndFor
\State
\State \textit{// compute predictions for $X^\text{test}$ using the first $\hat{\lambda}$ parameter indices in $L$ and remove duplicates}
\If{$\hat{\lambda} = k+1$}
    \State $C_{\hat{\lambda}}(X^\text{test}) \gets \{0,1\}^{W \times H}$ \textit{// set of all possible masks}
\Else
    \State $C_{\hat{\lambda}}(X^\text{test}) \gets \{ f(X^\text{test}, t_j) : j \in L[1:\hat{\lambda}]\}$
    \State $C_{\hat{\lambda}}(X^\text{test}) \gets \Call{Unique*}{C_{\hat{\lambda}}(X^\text{test}), \eta}$
\EndIf
\State \Return $C_{\hat{\lambda}}(X^\text{test})$
\end{algorithmic}
\end{multicols}
\end{algorithm} 

We present a modified version of our conformal algorithm that provides finite sample guarantees (Algorithm \ref{algorithm-finite-sample-guarantees}). 

We use the notation defined in Section \ref{sec:method} of the main text. For this version of the algorithm, the calibration dataset is randomly split into two sets of size $n_1$ and $n_2$: $\{(X_1,Y_1),\dots,(X_{n_1}, Y_{n_1})\}$ and $\{(X_{n_1+1},Y_{n_1+1}),\dots,(X_{n_1+n_2}, Y_{n_1+n_2})\}$. We use the first set of $n_1$ datapoints to rank the parameter values $t_1, \dots, t_k$ from highest to lowest priority via a greedy algorithm based on their performance on the data. Then, using the remaining $n_2$ datapoints, we apply Conformal Risk Control (CRC) \cite{angelopoulos2022conformal} to compute the smallest value $\hat{\lambda} \in [k]$ such that the first $\hat{\lambda}$ parameters in the ranked list result in a valid conformal prediction set. The CRC method provides finite sample guarantees for the probability that for a new input $X^\text{test}$ (from the same distribution as the calibration set), the conformal prediction set $C_{\hat{\lambda}}(X^\text{test})$ will contain a mask that has high IoU with the true mask. We describe the algorithm in more detail below. 

\paragraph{Ranking the parameters via greedy algorithm} We use the first set of $n_1$ datapoints as follows. Similar to Algorithm \ref{algorithm-conformal}, for each $t_j$, we first compute the set $S_j$ of data indices $i$ where the predicted mask $f(X_i, t_j)$ and true mask $Y_i$ have IoU $> \tau$. We say that $S_j$ is the set of data indices ``covered" by the parameter index $j$.
\begin{align*}
    S_j = \bigl\{ i : \iou \big(Y_i, f(X_i,t_j) \big) > \tau \bigr\} .
\end{align*}
Next, we create an ordered list $L$ of parameter indices from highest to lowest priority via a greedy algorithm. At each of $k$ steps, we append to $L$ the parameter index $\argmax_{j\in[k]} |S_{j}|$ that covers the greatest number of datapoints that have not already been covered, and then update each of $S_1, \dots S_k$ to remove the newly covered data indices. This approach is similar to the greedy set cover procedure \cite{johnson1973approximation, lovasz1975ratio, chvatal1979greedy} used in Line \ref{line:find_minimal_set_cover} of Algorithm \ref{algorithm-conformal} to obtain a set of parameter values that covers $(1-\alpha)$ fraction of the datapoints. However, in our new algorithm, we do not stop appending parameter indices to $L$ at a specific coverage fraction, but continue until all $k$ parameters have been ranked. 

\paragraph{Applying CRC to compute number of parameters to keep} Using the remaining $n_2$ datapoints, we apply CRC to compute the smallest $\hat{\lambda} \in [k]$ such that the first $\hat{\lambda}$ parameters in $L$ result in a valid conformal prediction set. For each $1 \leq \lambda \leq k$, we compute the empirical risk $\hat{r}(\lambda)$ as follows. We loop through the $n_2$ datapoints; for each point $(X_i, Y_i)$, we compute an ordered list of predictions $C_{\lambda}(X_i)$ using the first $\lambda$ parameters from $L$.
\begin{align*}
    C_{\lambda}(X_i) = \{ f(X_i, t_j) : j \in L[1:\lambda]\}.
\end{align*}
We remove duplicate predictions using a user-provided IoU threshold $0 < \eta < 1$. The duplicate removal procedure $\Call{Unique*}{}$ takes in an ordered list of predictions $C$ and recursively runs duplicate removal on the first $|C|-1$ elements of $C$; we call this intermediate output $C'$. If there exists $\hat{Y}' \in  C'$ such that $\hat{Y}'$ and the last element of $C$ have $\text{IoU} > \eta$, the procedure returns $C'$. Otherwise, the procedure returns the union of $C'$ and the last element of $C$. This ensures that every discarded prediction is within IoU $\eta$ of one that is kept, and no two predictions that are kept are within IoU $\eta$ of each other. We have 
\begin{align*}
    C_{\lambda}(X_i)  \gets \Call{Unique*}{C_{\lambda}(X_i), \eta}.
\end{align*}
We note that unlike the duplicate removal procedure $\Call{Unique}{}$ used in Algorithm \ref{algorithm-remove-duplicates}, $\Call{Unique*}{}$ ensures that $C_{\lambda}(X_i) \subseteq C_{\lambda+1}(X_i)$. This will allow us to construct a monotone loss function $l_i(\lambda)$, which is a necessary condition for CRC. For each $\lambda \in [k+1]$, we define the loss $l_i(\lambda)$ for the datapoint as 1 if none of the predictions in $C_{\lambda}(X_i)$ have IoU $> \tau$ with the true label $Y_i$, and 0 otherwise.
\begin{align*}
    l_i(\lambda) = \mathbbm{1}(\max_{\hat{Y}_i \in C_{\lambda}(X_i) } \text{IoU}(Y_i, \hat{Y}_i) \leq \tau).
\end{align*}
For each value of $\lambda$, we compute the empirical risk $\hat{r}(\lambda)$ by averaging the losses of the $n_2$ datapoints. This average is equivalent to the fraction of datapoints where none of the predictions (after duplicate removal) have sufficiently high IoU.
\begin{align*}
    \hat{r}(\lambda) = \frac{1}{n_2}\sum_{i=n_1+1}^{n_1+n_2}l_i(\lambda).
\end{align*}
If $\hat{r}(\lambda)$ is less than or equal to the CRC threshold of $\alpha - \frac{1-\alpha}{n_2}$, then we set $\hat{\lambda}$ equal to the current value of $\lambda$. Otherwise, we continue to iterate through the remaining values of $\lambda$ until we find a value where the empirical risk is less than or equal to the threshold. 

\paragraph{Compute conformal prediction set for new input} If there is no value of $\lambda \in [k]$ at which $\hat{r}(\lambda)$ is sufficiently low, then our algorithm returns $\{0,1\}^{W \times H}$, the set of all possible masks, so that $\hat{r}(k+1)=0$. (In practice, to avoid storing $2^{W \times H}$ masks, the algorithm would return an error message.) Otherwise, for a new test input $X^\text{test}$, we compute a set of predictions using the first $\hat{\lambda}$ parameters from $L$.
\begin{align*}
    C_{\hat{\lambda}}(X^\text{test}) = \{ f(X^\text{test}, t_j) : j \in L[1:\hat{\lambda}]\}.
\end{align*}
We remove duplicates to obtain the final conformal prediction set
\begin{align*}
    C_{\hat{\lambda}}(X^\text{test}) \gets \Call{Unique*}{C_{\hat{\lambda}}(X^\text{test}), \eta}.
\end{align*}

The resulting finite-sample guarantee is that, for a new test input $X^\text{test}$, with probability at least $1-\alpha$, there exists some $\hat{Y}^\text{test} \in C_{\hat{\lambda}}(X^\text{test})$ with $\text{IoU}(Y^\text{test}, \hat{Y}^\text{test}) > \tau$. This is stated formally in the following theorem.
\begin{theorem}
Let $C_{\hat{\lambda}}(\xt)$ denote the set of segmentation masks returned when running Algorithm \ref{algorithm-finite-sample-guarantees} with calibration samples $\big( (X_i,Y_i) \big)_{i=1}^{n_1}$ and $\big( (X_i,Y_i) \big)_{i=n_1+1}^{n_1+n_2}$ and parameters $\alpha,\tau,\eta \in (0,1)$. Under Assumption \ref{assump:IID} (calibration and test samples are i.i.d.), $$\mathbb{P} \Big(  \max_{\hat{y} \in C_{\hat{\lambda}}(\xt) } \bigl\{ \iou( \yt,\hat{y}) \bigr\} > \tau \Big) \geq 1-\alpha.$$
\end{theorem}
\begin{proof} 
For each calibration point $(X_i, Y_i)$, the loss function $l_i(\lambda) = \mathbbm{1}(\max_{\hat{Y}_i \in C_{\lambda}(X_i) } \text{IoU}(Y_i, \hat{Y}_i) \leq \tau)$ is non-increasing in $\lambda$. This is because by design $C_{\lambda}(X_i) \subseteq C_{\lambda+1}(X_i)$, so the maximum IoU of the prediction set can only improve as $\lambda$ increases. Furthermore, we define $C_{k+1}(X_i)$ as the set of all possible masks $\{0,1\}^{W \times H}$, which is guaranteed to contain the true mask, so $l_i(\lambda_{\text{max}}) = l_i(k+1) = 0 < \alpha$. Thus, we can apply Conformal Risk Control \cite{angelopoulos2022conformal}.

Our algorithm computes the smallest value of $\lambda$ such that the empirical risk $\hat{r}(\lambda) = \frac{1}{n_2}\sum_{i=n_1+1}^{n_1+n_2}l_i(\lambda)$ falls below the CRC threshold:
\begin{align*}
    \hat{\lambda} = \min\{\lambda: \hat{r}(\lambda) \leq \alpha - \frac{1-\alpha}{n_2}\}.
\end{align*}

For a new test point $(\xt, \yt)$, let $l_{\text{test}}(\lambda) = \mathbbm{1}(\max_{\hat{y} \in C_{\lambda}(\xt) } \text{IoU}(\yt, \hat{y}) \leq \tau)$. Then by Theorem 1 of \citet{angelopoulos2022conformal}, 
\begin{align*}
    \mathbb{P} \Big(  \max_{\hat{y} \in C_{\hat{\lambda}}(\xt) } \bigl\{ \iou( \yt,\hat{y}) \bigr\} \leq \tau \Big) = \mathbb{E}[l_{\text{test}}(\hat{\lambda})] \leq \alpha.
\end{align*}
\end{proof}
We note that the parameter-ranking step using the first $n_1$ datapoints is not necessary for our algorithm to be valid, since CRC can be applied regardless of how $L$ is ordered. However, if we do not rank the parameters based on their performance on the data, $\hat{\lambda}$ may be very large, i.e. we may need to keep many parameters in order to ensure sufficiently low empirical risk; this is computationally costly and may result in large prediction sets that are inconvenient to use. 

\paragraph{Comparison to asymptotic-guarantee algorithm} The finite-sample-guarantee version of the algorithm will generally give similar results as the asymptotic-guarantee version presented in Section \ref{sec:method} of the main text. This is because the greedy algorithm used in the parameter-ranking step of Algorithm \ref{algorithm-finite-sample-guarantees} is closely related to the greedy set cover algorithm used in Line \ref{line:find_minimal_set_cover} of Algorithm \ref{algorithm-conformal}. In Algorithm \ref{algorithm-conformal}, we use the greedy set cover algorithm to find a set of parameters that covers $(1-\alpha)$ fraction of the calibration points; analogously, in the CRC step of Algorithm \ref{algorithm-finite-sample-guarantees}, we add parameters from the ranked list $L$ until we cover $(1-\alpha+\frac{1-\alpha}{n_2})$ fraction of the $n_2$ calibration points. 

However, one advantage of Algorithm \ref{algorithm-conformal} is that we can improve the greedy set cover by brute-forcing over all smaller subsets of $[k]$ in order of cardinality until we either find a true minimal-size set cover or confirm that the greedy set cover is already minimal. This is not possible in Algorithm \ref{algorithm-finite-sample-guarantees} because we are restricted to parameter sets of the form $L[1:\lambda]$, i.e., adding parameters in order of their greedy algorithm ranking. Thus, Algorithm \ref{algorithm-finite-sample-guarantees} cannot guarantee that the set of parameters used to produce the conformal prediction set is minimal-size. 

Another advantage of the asymptotic algorithm is that the duplicate removal procedure $\Call{Unique}{C, \eta}$ used in Algorithm \ref{algorithm-remove-duplicates} brute-forces over all subsets of $C$ in order of increasing cardinality until it finds a \textit{minimal-size} subset such that every discarded prediction is within IoU $\eta$ of one that is kept. In contrast, the duplicate removal procedure $\Call{Unique*}{C, \eta}$ used in Algorithm \ref{algorithm-finite-sample-guarantees} outputs a subset of predictions such that every discarded prediction is within IoU $\eta$ of one that is kept and no two predictions that are kept are within IoU $\eta$ of each other, but this subset is not necessarily minimal-size. ($\Call{Unique*}{C, \eta}$ ensures that even after duplicate removal, $C_{\lambda}(X_i) \subseteq C_{\lambda+1}(X_i)$ and thus the maximum IoU-based loss function $l_i(\lambda)$ is non-increasing in $\lambda$, which is a necessary condition for CRC.) As a result, the final prediction sets produced by Algorithm \ref{algorithm-finite-sample-guarantees} may be larger than those from the asymptotic algorithm.

We note that the finite-sample algorithm has the advantage of computational efficiency because its greedy parameter ranking procedure and duplicate removal procedure run in polynomial time, making it useful in cases where brute-forcing is not computationally tractable.

\newpage
\section{Datasets and Models}\label{appendix:datasets}
\paragraph{Agricultural field delineation} Delineating agricultural fields from satellite imagery is an important task for monitoring land use and supporting precision agriculture. The state-of-the-art uses deep semantic segmentation models, such as U-Nets and their variants, followed by post-processing steps like watershed segmentation or connected components analysis to separate adjacent fields \cite{waldner2021detect, kerner2025fields}. 

Each satellite image contains multiple field instances. We use 350 images and instance segmentation labels from the Fields of the World (FTW) test dataset for France \cite{kerner2025fields},  
and we randomly divide the data into 250 calibration images and 100 test images. 
We randomly sample 50 pixels from each image and only keep pixels that are in a field, resulting in 2476 calibration pixels and 917 test pixels.

For predictions, we use the pretrained FTW U-Net model provided by \citet{kerner2025fields} to estimate field boundary probabilities at each pixel, followed by a watershed algorithm to convert the probability maps into field instances. The tunable parameter $T$ corresponds to the watershed threshold: as $T$ increases, neighboring fields are merged, yielding larger predicted instances. We consider $T$ values from 0 to 0.009 in steps of 0.001, and from 0.01 to 1 in steps of 0.01. 

\paragraph{Cell segmentation} Segmenting individual cells from microscopy images is a common task in biology and medicine, supporting cell counting and morphology characterization. A widely used approach is Cellpose \cite{stringer2021cellpose}, which has recently been extended with Segment Anything (Cellpose-SAM) \cite{pachitariu2025cellpose} to improve robustness across diverse imaging conditions.

We evaluate our method on 52 grayscale images with instance segmentation labels from the Cellpose test dataset. 
We split the data into 30 calibration images and 22 test images and randomly sample 50 pixels from each image, only keeping pixels in a cell (or other object) instance. This results in 925 calibration pixels and 659 test pixels.

For predictions, we use the Cellpose-SAM model to segment each image into cell instances. The tunable parameter $T$ corresponds to the ``cell extent'' threshold: increasing $T$ causes more pixels to be classified as background, producing smaller predicted instances. We sweep $T$ from –5 to 5 in steps of 1. 

\paragraph{Vehicle detection} Detecting vehicles in street-level imagery is a central task in autonomous driving and traffic monitoring.

We evaluate our method on 200 images with car instance labels from the Cityscapes street scene validation dataset \cite{cordts2016cityscapes}. We split the data into 100 calibration and 100 test images and randomly sample 20 pixels from each image, only keeping pixels in the \emph{car} class. This yields 105 calibration and 121 test pixels.

For predictions, we use the Segment Anything Model (SAM) \cite{kirillov2023segment}, which takes an RGB image and a pixel coordinate and returns mask probability maps for the object containing that pixel. SAM produces three candidate mask probability maps per query, ranked by confidence. There are two tunable parameters: $T_1$, which indexes the candidate mask ($1,2,3$), and $T_2$, the mask probability threshold. As $T_2$ increases, more pixels are classified as not being in the mask, so the predicted object instances become smaller. We iterate over $T_1$ values in $\{1, 2, 3\}$ and $T_2$ values from 0 to 1 at increments of 0.05. 

\newpage
\section{Experimental validation of finite sample version of algorithm}\label{appendix:finite-sample-experiments}

We perform an empirical comparison of the asymptotic and finite sample guarantee versions of our conformal algorithm. We run the finite sample algorithm (Algorithm \ref{algorithm-finite-sample-guarantees}) using the same error rate $\alpha$ and IoU threshold $\tau$ as in the original experiments from the main text, as well as the same calibration and test sets. The finite sample algorithm requires randomly splitting the calibration pixels into two sets; we use a 50-50 split. 

We evaluate the coverage of the finite sample algorithm on the test set and compare it to the coverage of the asymptotic version (Table \ref{table:finite-sample-version}). Similar to the asymptotic version, the finite sample coverages are close to the target coverages $(1-\alpha)$, as expected. We note that the finite sample coverages are computed using the IoU threshold $\tau$ instead of $\tilde{\theta}$ (which is used for the asymptotic version) since the finite sample algorithm does not need to re-calibrate the IoU threshold after duplicate removal.

\renewcommand{\tabcolsep}{5pt}
\begin{table}[h]
\caption{\textbf{Asymptotic and Finite Sample Algorithm Coverages of Test Set Predictions.}}\label{table:finite-sample-version}
\begin{center}
\tiny
\begin{tabular}{p{1.3cm}p{0.3cm}p{0.3cm}p{1.7cm}p{0.8cm}p{1.8cm}}
\toprule
EXAMPLE & $\tau$ & $\tilde{\theta}$ & TARGET COVERAGE \newline $(1-\alpha)$ & ASYMPTOTIC \newline COVERAGE & FINITE-SAMPLE \newline COVERAGE \\ \midrule 
Field delineation & 0.7 & 0.696 & 0.8 & 0.797 & 0.779 \\ \midrule
Cell segmentation & 0.75 & 0.760 & 0.8 & 0.835 & 0.816 \\ \midrule
Vehicle detection & 0.8 & 0.842 & 0.9 & 0.843 & 0.893 \\ \bottomrule
\end{tabular}
\end{center}
\end{table}

Furthermore, we visualize the distribution of conformal prediction set sizes for test data points after removing duplicate predictions (Figure \ref{fig:finite-sample-conformal-set-sizes}). Like the asymptotic algorithm, the finite sample algorithm results in diverse set sizes after duplicate removal, adapting to query difficulty (although the maximum prediction set sizes differ slightly for the asymptotic and finite sample algorithms in the field delineation and cell segmentation examples).

\begin{figure*}[h!]
  \centering
  \includegraphics[width=0.8\textwidth]{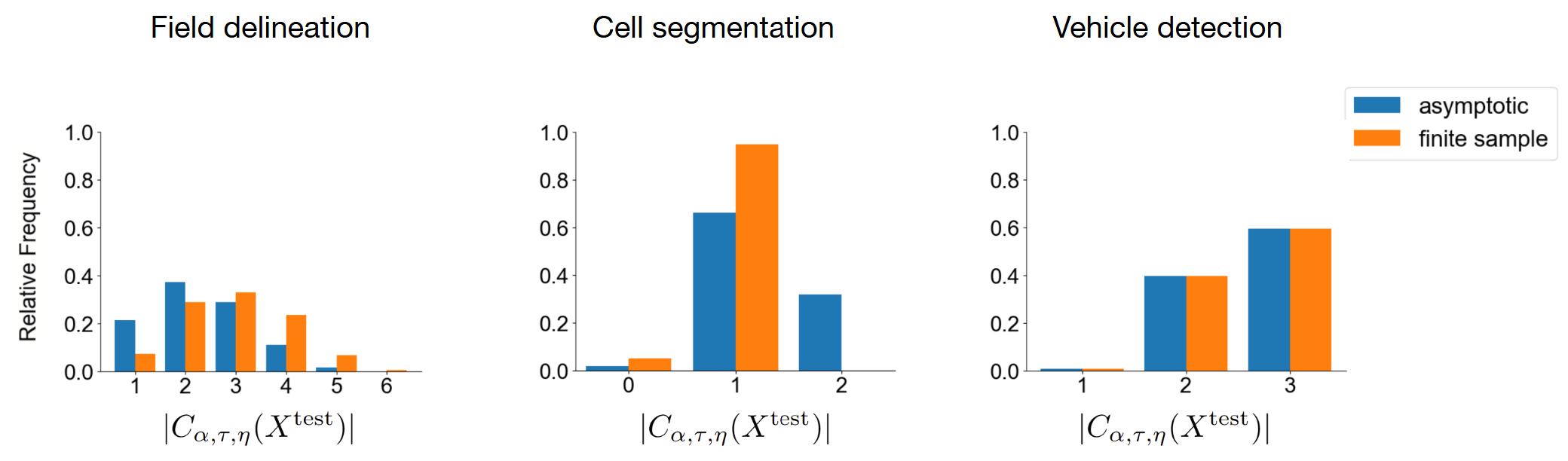}
  \caption{\textbf{Distribution of conformal prediction set sizes on test set for asymptotic and finite sample versions of the algorithm.}}\label{fig:finite-sample-conformal-set-sizes}
\end{figure*}

\newpage
\section{Sensitivity analyses on conformal algorithm parameters}\label{appendix:sensitivity-analyses}
We run experiments to analyze the robustness of our algorithm, including sensitivity analyses on error rate $\alpha$, IoU threshold $\tau$, duplicate removal IoU threshold $\eta$, tunable parameter space size $k$, and calibration set size $n$ (Figure \ref{fig:sensitivity-analyses}). 

\begin{figure*}[h!]
  \centering
  \includegraphics[width=0.75\textwidth]{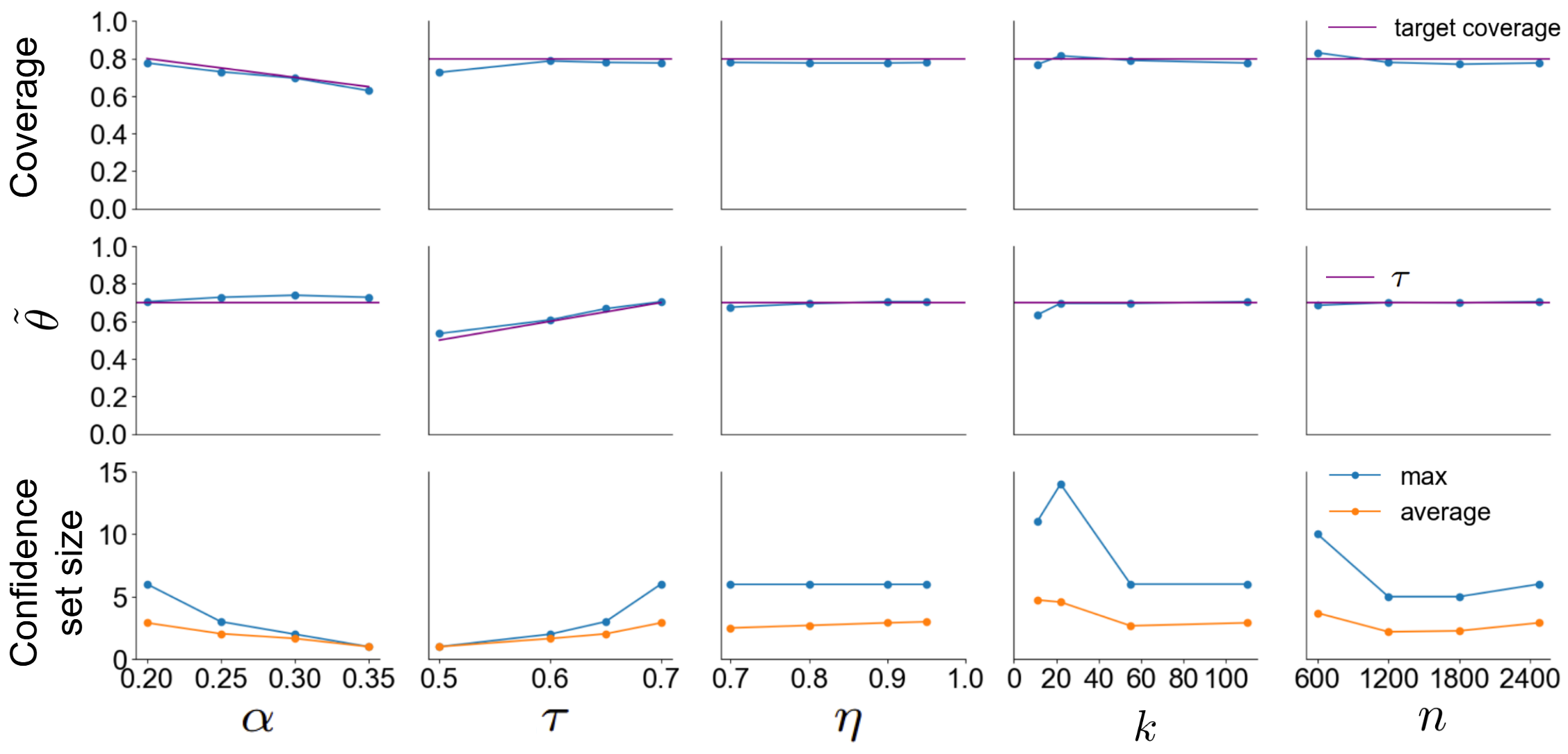}
  \caption{\textbf{Sensitivity analyses on conformal algorithm parameters: error rate $\alpha$, IoU threshold $\tau$, duplicate removal IoU threshold $\eta$, tunable parameter space size $k$, and calibration set size $n$.}}\label{fig:sensitivity-analyses}
\end{figure*}

We use the field delineation dataset with the same default parameters as in the main text: $\alpha=0.2$, $\tau=0.7$, $\eta=0.9$, $k=110$, and $n=2476$. For each experiment, we vary one parameter while holding the others constant. We use the calibration set to compute a confidence set for the tunable parameter, re-calibrate the IoU threshold after duplicate removal from $\tau$ to $\tilde{\theta}$, and evaluate performance on the test set of 917 points. For each experiment, we record $\tilde{\theta}$, the test coverage, maximum prediction set size, and average prediction set size on the test set. (For all experiments, we use the asymptotic-guarantee version of our algorithm with the greedy set cover subprocedure for computational efficiency.)

\paragraph{Varying error rate $\alpha$} We vary error rate $\alpha = 0.2, 0.25, 0.3, 0.35$. For all experiments, coverage is close to the target coverage of $1-\alpha$, showing that our method is robust to changes in $\alpha$. The re-calibrated IoU threshold $\tilde{\theta}$ stays close to $\tau=0.7$ as expected. As $\alpha$ increases, the prediction set size decreases because fewer predictions are required in order to attain the desired $1-\alpha$ coverage. 

\paragraph{Varying IoU threshold $\tau$} We vary IoU threshold $\tau = 0.5, 0.6, 0.65, 0.7$. For $\tau=0.5$, the coverage (0.73) is slightly lower than the target coverage of $1-\alpha=0.8$; we believe this is due to random variation between the calibration and test sets. For all other experiments, coverage is close to the target coverage. The re-calibrated IoU threshold $\tilde{\theta}$ stays close to $\tau$ as expected, for all values of $\tau$. As $\tau$ increases, the prediction set size increases because more predictions are required in order to attain IoU greater than $\tau$ on at least one prediction in the set. 

\paragraph{Varying duplicate removal IoU threshold $\eta$} We vary duplicate removal IoU threshold $\eta = 0.7, 0.8, 0.9, 0.95$. For all experiments, coverage is close to the target coverage of $1-\alpha=0.8$, showing that our method is robust to changes in $\eta$. The re-calibrated IoU threshold $\tilde{\theta}$ stays close to $\tau$; even for $\eta$ as low as 0.7, $\tilde{\theta}$ remains fairly high (0.675), so duplicate removal only has a minor effect on whether the final prediction set contains a high enough IoU prediction. The maximum prediction set size stays the same because it is calibrated using $\alpha$ and $\tau$ only and does not depend on $\eta$ (see Algorithm \ref{algorithm-conformal}). However, the average prediction set size increases from 2.5 to 3.0 because fewer elements are removed during duplicate removal.

\paragraph{Varying tunable parameter space size $k$} We vary tunable parameter space size $k = 11, 22, 55, 110$. (In the paper, we use $k=110$ tunable parameter values from 0 to 0.009 in steps of 0.001, and from 0.01 to 1 in steps of 0.01. For the sensitivity analysis, we create sets of $k = 11, 22, 55$ tunable parameter values by increasing the step sizes by a factor of 10, 5, and 2 respectively.) For $k=11$, no subset of the parameters could attain the target IoU threshold $\tau=0.7$ on $1-\alpha=0.8$ fraction of calibration points, so we used the full set of 11 parameters in lieu of a confidence set. This resulted in a lower re-calibrated IoU threshold ($\tilde{\theta}=0.63$). The re-calibrated IoU threshold $\tilde{\theta}$ stays close to $\tau=0.7$ for all other values of $k$. For all experiments, coverage at $\tilde{\theta}$ is close to the target coverage of $1-\alpha=0.8$. Finally, when $k$ is small, the maximum and average prediction set sizes are large because with a less granular tunable parameter space, more parameters are needed in order to attain high maximum IoU (and when $k=11$, we are unable to attain IoU $>0.7$ on 80\% of the calibration points even if we use all the parameters).

\paragraph{Varying calibration set size $n$} We vary calibration set size $n = 600, 1200, 1800, 2476$. For all experiments, coverage is close to the target coverage of $1-\alpha=0.8$. The re-calibrated IoU threshold $\tilde{\theta}$ stays close to $\tau=0.7$ as expected. However, for the smallest value $n=600$, the maximum and average prediction set sizes are substantially larger than for the other values of $n$, indicating that the calibration results are less stable at small calibration set sizes; the computed confidence set of tunable parameter values is more sensitive to random variation in the calibration data when $n$ is small. When we tried further decreasing $n$ to 300, no subset of the parameters could attain IoU $>0.7$ on 80\% of calibration points.

\end{document}